\documentclass{article} 
\usepackage{iclr2027_conference,times}

\usepackage{amsmath,amsfonts,bm}

\def\eqref#1{equation~\ref{#1}}

\def\plaineqref#1{\ref{#1}}

\def\1{\bm{1}}

\def\vp{{\bm{p}}}

\def\vx{{\bm{x}}}

\DeclareMathAlphabet{\mathsfit}{\encodingdefault}{\sfdefault}{m}{sl}
\SetMathAlphabet{\mathsfit}{bold}{\encodingdefault}{\sfdefault}{bx}{n}

\newcommand{\E}{\mathbb{E}}

\newcommand{\R}{\mathbb{R}}

\newcommand{\eddy}{\textsc{EDDY}}
\newcommand{\eddyrbf}{\textsc{EDDY-RBF}}
\newcommand{\pg}{\textsc{PG}}
\newcommand{\cads}{\textsc{CADS}}

\newcommand{\df}{\textsc{DF}}
\newcommand{\base}{\textsc{I.I.D}}

\newcommand{\N}{\mathcal{N}}
\DeclareMathOperator{\divergence}{div}

\newcommand{\Ap}{\mathcal{A}_p}
\newcommand{\Apt}{\mathcal{A}_{p_t}}
\newcommand{\score}{\nabla \log p}
\newcommand{\Weiner}{W_t}
\newcommand{\drift}{\mu}
\newcommand{\vol}{\sigma}

\newcommand{\dd}{d}
\newcommand{\np}{n}
\newcommand{\fdeps}{\epsilon}
\newcommand{\wg}{w_g}

\newcommand{\xbar}{\bar{x}}
\newcommand{\xap}[2][0]{\tilde{x}^{#2}_{#1}}

\newcommand{\xt}[1][t]{x_{#1}}
\newcommand{\Xt}[1][t]{\vx_{#1}}

\newcommand{\si}{s^{(i)}}
\newcommand{\vj}[1][j]{v^{(#1)}}

\newcommand{\muj}[1][j]{\mu^{(#1)}}
\newcommand{\psii}{\psi^i}
\newcommand{\bbias}{b^{(i)}}
\newcommand{\xp}[1]{x^{#1}}
\newcommand{\xpt}[2][t]{x^{#2}_{#1}}
\newcommand{\xpnit}[2][t]{x^{-#2}_{#1}}

\newcommand{\Ai}{A^{(i)}}

\newcommand{\dij}{\delta^{(ij)}}
\newcommand{\kij}{k(\xp{i}, \xp{j})}
\newcommand{\rij}{r^{(ij)}}
\newcommand{\Kij}{K^{(ij)}}
\newcommand{\Aj}{A^{(ij)}}

\usepackage{hyperref}
\usepackage{url}

\usepackage{graphicx}
\usepackage{multirow}
\usepackage{tabularx}
\usepackage{subcaption}
\usepackage{booktabs}
\usepackage{algorithm}
\usepackage[algo2e]{algorithm2e}
\SetKw{KwBy}{by}

\usepackage{xcolor}
\usepackage{cancel}
\usepackage{amssymb}
\usepackage{amsthm}

\newtheorem{claim}{Claim}

\newtheorem{remark}{Remark}
\theoremstyle{definition}
\newtheorem{definition}{Definition}

\usepackage{titletoc}
\newcommand\DoToC{%
  \startcontents
  \printcontents{}{1}{\textbf{Table of Contents}\vskip3pt\hrule\vskip5pt}
  \vskip3pt\hrule\vskip5pt
}

\title{Diverse Sampling in Diffusion Models with \\ Divergence-Free Particle Guidance}

\author{ Gal Vinograd \\ Faculty of Engineering \\ Bar-Ilan University\\ Ramat
Gan, Israel \\ gal.vinograd@biu.ac.il \\ \And Idan Achituve \\
Independent Researcher \\ idanachi@gmail.com \\ \And Ethan Fetaya \\ Faculty
of Engineering \\ Bar-Ilan University\\ Ramat Gan, Israel \\ ethan.fetaya@biu.ac.il
}

\iclrfinalcopy 
\begin{document}

\maketitle

\begin{abstract} 
Modern generative models can produce high-quality samples, but independent samples under the same condition often yields highly similar outputs. Particle-based methods increase diversity by letting the samples in a batch interact, typically through a repulsive force. These forces, however, also push each sample away from the data distribution and produce visible artifacts. We introduce EDDY, a training-free particle guidance method designed to avoid this failure mode. Instead of a repulsive gradient, EDDY couples particles through anti-symmetric matrix fields passed through a Stein operator, a family of drift perturbations that leaves the Fokker--Planck equation invariant. Unlike repulsive guidance, these interactions do not alter a particle's distribution while the batch is independent. For the same reason they cannot create diversity on their own, so EDDY pairs them with a negatively correlated initialization that keeps each particle's prior exact. To use EDDY with perceptual kernels such as DINOv2, we approximate its second-order terms with finite differences and Hutchinson estimates. Across FLUX.1-dev, FLUX.2-klein and SDXL, EDDY achieves higher image quality and prompt alignment than existing particle guidance methods at matched diversity. 

\end{abstract}

\section{Introduction}
Diffusion and flow matching models \citep{albergo2023building,ho_denoising_2020,lipman_flow_2022, liu_flow_2022, sohl2015deep,song_score-based_2021} have become a prominent class of generative models, enabling high-quality sampling across multiple modalities, such as images \citep{ dhariwal2021diffusion, lipman_flow_2022, rombach2022high,song2021denoising}, molecules \citep{corso2023diffdock,hoogeboom2022equivariant,xu2022geodiff, jing2022torsional}, and texts \citep{austin2021structured,li2022diffusion, lou2024discrete, sahoo2024simple}. These models can be viewed as transporting a simple prior distribution to a complex target distribution through time-dependent stochastic or deterministic dynamics. One advantage of this formulation is that generation can be modified at inference time through guidance \citep{ho_classifier-free_2021,rombach2022high}: the drift can be adjusted to favor a condition, satisfy a constraint, or steer samples toward desired properties. However, when multiple samples are drawn for the same condition, the resulting set may contain similar outputs.

This issue has motivated a growing body of work on diverse sampling \citep{kynkaanniemi2024applying, lu2024procreate,sadat_cads_2024}. In many applications, one does not want a single output, but rather a small set of samples that explores different plausible modes of the conditional distribution. Naively drawing independent samples is often inefficient, partly because guidance improves sample fidelity at the expense of diversity \citep{ho_classifier-free_2021}. Recent methods addressed this issue by introducing interactions between samples, for example, by optimizing initial noise latents \citep{ harrington2025s,kim_diverse_2025} or adding repulsive particle potentials during sampling \citep{corso_particle_2023, jalali2026sparke,morshed_diverseflow_2025}. These approaches can improve diversity but can suffer from severe drawbacks. Mainly, we found that adding repulsive terms during sampling creates noticeable artifacts, while optimizing the noise latents had a limited effect on diversity (see examples in Fig. \ref{fig:cheff_grid}).

In this work, we ask why particle-based methods produce artifacts, and whether the interactions can be designed to avoid them. Repulsive methods add the gradient of a similarity potential to the drift. Even when the particles are independent, such a gradient moves probability mass away from the data distribution (Figure~\ref{fig:teaser}), so every sample is pushed off the data manifold whether or not the batch needs to spread out. We look instead for interactions whose effect on a sample's distribution vanishes when the particles start as independent. The Fokker-Planck equation that governs the density evolution provides such a family: drift perturbations obtained by applying a Stein-type operator to anti-symmetric matrix fields change individual trajectories without changing the density evolution. EDDY builds its pairwise interactions from this family. For each pair of particles, it forms an anti-symmetric matrix from the difference of two outer products involving the kernel's repulsive direction, and averages the resulting perturbations over the batch.

However, the property that protects an marginals also keeps it independent: starting from i.i.d.\ samples, the joint distribution remains a product, so the interaction adds no diversity (Claim ~\ref{thm:product_propagation}). EDDY therefore starts from a negatively correlated prior, which keeps each particle's initial distribution exactly Gaussian.  Once the particles are coupled, the interaction can generate diversity but also shift the marginal. We  then show empirically that our construction allows us to achieve better diversity-distortion tradeoffs relative to other methods.

We further observe that a substantial part of the diversity EDDY produces in practice comes from an unexpected source: the discretization of the sampler (see Figure~\ref{fig:eddy_disc} in the appendix). For independent particles, the guidance's contribution to the density evolution cancels exactly in continuous time, but a finite sampling step keeps this cancellation only to first order in the step size. The remainder involves pairs of particles, since each particle's guidance depends on its neighbors, so it makes the particles dependent; empirically, this dependence increases diversity. With the few, large steps used in text-to-image sampling, this remainder is not negligible. It can also shifts each particle's marginal, however we observe a very favorable tradeoff.

\begin{figure}[t]
    \centering
    \includegraphics[width=0.7\linewidth]{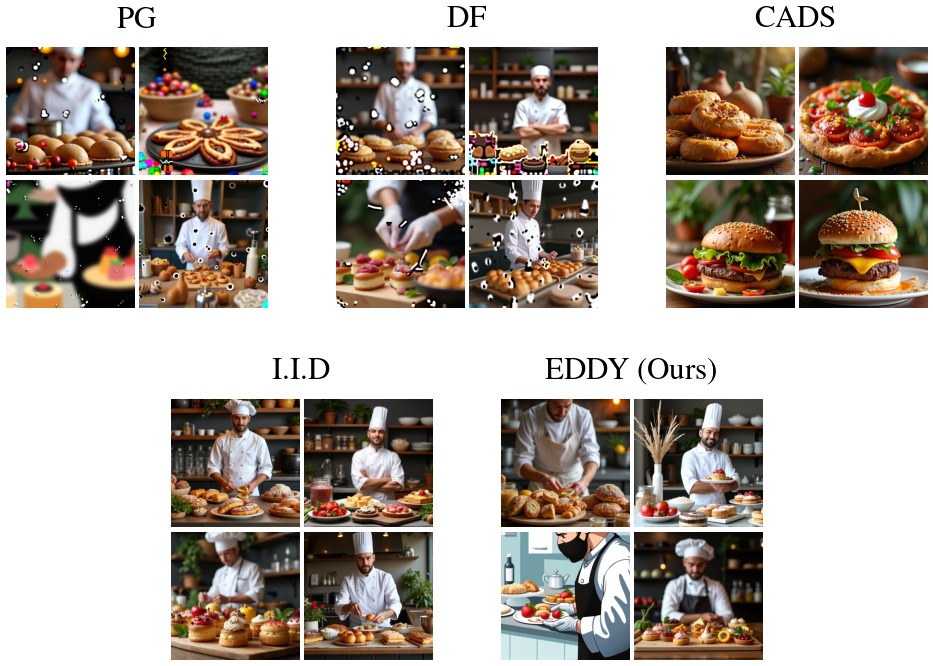}
    \caption{
        Qualitative comparison on the prompt \emph{``a chef's display of ingredients and pastry creations''}
        \pg{} (top-left) and \df{} (top-middle): Increase in variety with visible artifacts.
        \cads{} (top-right): High-quality samples but with low prompt adherence (no chef or ingredients in any image, burgers instead of pastries). 
        \base{} (bottom-left): Low diversity between samples.
        \eddy{} (bottom-right): Diverse output that preserves quality and prompt fidelity.
    }
    \label{fig:cheff_grid}
\end{figure}

For practical text-to-image generation, we often wish to define similarity not directly in pixel space, but in a perceptual feature space such as DINO or CLIP embeddings. In this setting, computing the 
exact kernel can be prohibitively expensive due to second order derivatives.
We therefore approximate the expensive elements with finite differences and Hutchinson trace estimation. These approximations add a third source of distortion, on top of the coupling and the discretization; empirically, EDDY still achieves higher sample quality than competing methods at matched diversity.

We make the following contributions: 
(i) \textbf{EDDY}, a training-free particle guidance method that builds pairwise interactions from anti-symmetric Stein perturbations, together with a negatively correlated initialization that preserves each particle's prior. (ii) \textbf{A scalable approximation} that uses only kernel evaluations and first-order feature gradients,  running with perceptual kernels such as DINOv2 at 10--15\% overhead on FLUX. (iii) \textbf{An empirical evaluation} on three text-to-image backbones showing a better diversity--quality trade-off than PG and DiverseFlow at matched diversity, plus a synthetic study where the target is known. (iv) \textbf{An analysis that motivates the design}: anti-symmetric interactions cannot correlate an independent batch, which is why EDDY needs a coupled initialization; and unlike repulsive guidance, their distortion of the marginal arises only through the coupling.

\begin{figure}[t]
    \centering
    \includegraphics[width=0.7\linewidth]{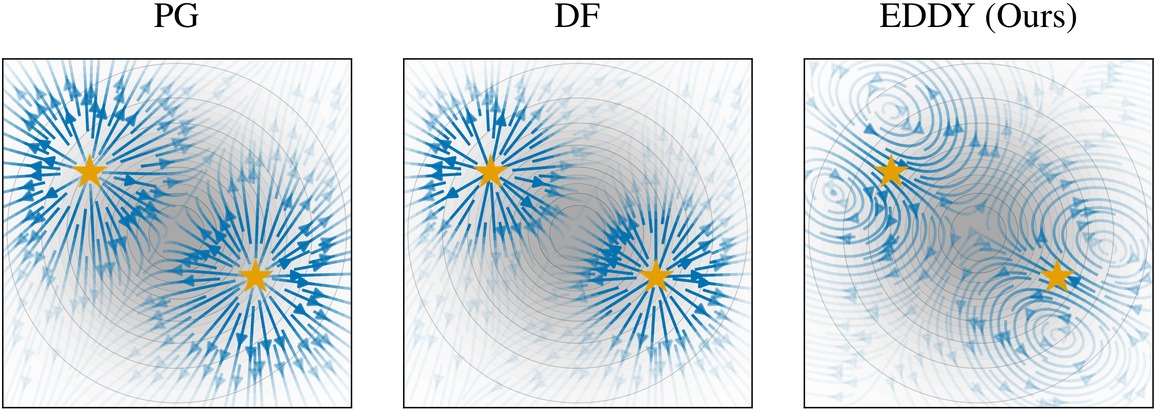}
    \caption{%
        Transport fields induced by \pg{} (left), \df{} (middle) and \eddy{} (right) for a standard Gaussian target (gray contours): how two particles (stars) affect a third test particle at each location.
        \pg{} and \df{} applies the mean negative RBF kernel gradient, a purely repulsive field that distorts the sampling distribution even between independent particles.
        \eddy{} applies a divergence-free field.
    }
    \label{fig:teaser}
\end{figure}

\section{Background}
Throughout the paper, state variables such as $\xt \in \R^d$ and functions such as $\mu_t$ and $p_t$ carry implicit argument $\xt$ unless stated otherwise; and $\nabla$ ($\divergence$) denotes the gradient (resp. divergence) with respect to $\xt$.

\subsection{Diffusion Models and Flow Matching}

Modern generative modeling transports a simple prior $p_0 = \N(0, I_d)$ to a target $p_1$ over $t\in[0,1]$ through continuous-time dynamics.
DDPM~\citep{ho_denoising_2020,song_score-based_2021} fixes a Gaussian noising process with noise schedule $\beta_t$ and samples $\xt[1] \sim p_1$ by integrating the time-reversal SDE~\citep{anderson_reverse-time_1982},
\begin{equation}
    d\xt = \left[\tfrac{1}{2}\beta_t\,\xt + \beta_t\,\score_t(\xt)\right] dt + \sqrt{\beta_t}\, d\Weiner,
    \label{eq:reverse_sde}
\end{equation}
forward in $t$, where $\Weiner$ is the Wiener process.
Optimal-transport flow matching (OT-FM)~\citep{albergo2023building,lipman_flow_2022,liu_flow_2022} instead integrates the deterministic ODE
\begin{equation}
    \label{eq:otfm_drift}
    \frac{\partial}{\partial t}\xt = \E[\xt[1] - \xt[0] \mid \xt].
\end{equation}
Both frameworks therefore fall under a common It\^o SDE
\begin{equation}
    d\xt = \drift_t(\xt)\, dt + \vol_t\, d\Weiner,
    \label{eq:forward_sde}
\end{equation}
with drift $\drift : \R^d \times [0,1] \to \R^d$ and volatility $\vol : [0,1] \to \R_{\geq 0}$: diffusion takes $\vol_t > 0$ and absorbs the score into the drift as in~\eqref{eq:reverse_sde}, while OT-FM has $\vol \equiv 0$ and learns $\drift_t$ directly.
We assume throughout that any condition (e.g.\ a prompt) is already incorporated into the drift.
The score $\score_t(\xt)$ is available in both: diffusion models output it up to a known scaling factor, and in flow matching it follows from $\mathbb{E}[\xt[1] - \xt[0] \mid \xt]$ by Tweedie's formula~\citep{efron_tweedies_2011},
\begin{equation}
    \score_t(\xt) = \frac{t\E[\xt[1]-\xt[0]\mid\xt] - \xt}{1-t}.
\label{eq:tweedie_flow}
\end{equation}
see Appendix~\ref{app:tweedie} for the derivation.

\subsection{Fokker--Planck Equation and its Symmetries}
\label{sec:symmetries}

The state $\xt$ in~\eqref{eq:forward_sde} is a random variable whose marginal density $p_t$ on $\R^d$ evolves deterministically under the Fokker-Planck equation (FPE),
\begin{equation}
    \frac{\partial}{\partial t} p_t \;=\; -\divergence (p_t\,\drift_t) \;+\; \frac{\vol_t^2}{2}\,\Delta p_t,
    \label{eq:fpe}
\end{equation}
so each pair $(\drift, \vol)$ induces a path of marginals $\{p_t\}_{t\in[0,1]}$ from $p_0$ to $p_1$.
This mapping is not one-to-one: many $(\drift, \vol)$ produce the {identical} evolution of $p_t$.
These are the \emph{symmetries} of the FPE: changes that leave the density evolution, and hence every marginal, invariant while reshaping how individual particles move.

\begin{definition}[Matrix Stein operator]
\label{def:Ap_matrix}
For a matrix field $F : \R^\dd \to \R^{\dd\times \dd}$, the Stein operator \citep{stein_bound_1972} over $F$ is defined as
\begin{equation}
    \label{eq:Ap_matrix}
    \Ap(F) \;=\; \frac{1}{p}\divergence\left(pF\right) \;=\; \divergence F \;+\; F\,\nabla \log p,
\end{equation}
where $\divergence F$ is the row-wise divergence: $(\divergence F)_i = \sum_j \frac{\partial}{\partial x_j}F_{ij}$.
\end{definition}

\begin{claim}[Symmetry of the FPE]
\label{claim:rot_symmetry}
Let $A : \R^\dd \to \R^{\dd \times \dd}$ be any smooth anti-symmetric matrix field (i.e., $A^\top = -A$), that can depend on $t$.
Then the guidance $\psi_t = \Apt(A)$ yields the same Fokker--Planck equation~\plaineqref{eq:fpe} and hence the same marginal $p_t$ for all $t$~\citep{hua_simulation-free_2025,richter-powell_neural_2022}.
\end{claim}

\begin{proof}
Substituting the guided drift $\tilde{\drift}_t = \drift_t + \psi_t$ into \eqref{eq:fpe}, the additional contribution is
$\divergence(p_t\, \Ap(A)) = \divergence(p_t\nabla(p_t A )/p_t) = \divergence(\divergence(A\,p_t))$ where the inner divergence acts on a matrix and the outer divergence acts on a vector.
For any anti-symmetric $A$, this double divergence vanishes identically by the symmetry of mixed partial derivatives ($\tfrac{\partial^2}{\partial x_i \partial x_j} (A_{ij} p_t) + \tfrac{\partial^2}{\partial x_j \partial x_i} (A_{ji} p_t) = 0$ since $A_{ji} = -A_{ij}$).
\end{proof}

The operator $\Ap$ in~\eqref{eq:Ap_matrix} is the classical Stein operator $f \mapsto \divergence(f) + \langle f, \nabla \log p \rangle$, but acting on each row of the matrix $F$ independently.

\section{Related Work}
\label{sec:related}
\textbf{Diverse Conditional Sampling in Diffusion Models.} Guidance is a powerful mechanism to gain sampling control in diffusion and flow models. Early methods in this area of study are classifier guidance \citep{dhariwal2021diffusion} and classifier-free guidance (CFG) \citep{ho_classifier-free_2021}, both trade-off diversity for fidelity. Several studies improve diversity while maintaining high fidelity in CFG. Inspired by the forward process in diffusion models, \cads{} \citep{sadat_cads_2024} injects random noise into the conditioning signal. This increases sample diversity while maintaining strong image quality; however, the noisy conditioning leads to a noticeable decline in prompt alignment. \citet{kynkaanniemi2024applying} limit guidance to a middle interval of sampling steps. \citet{koulischer2026feedback} adapt the guidance scale from the posterior likelihood of the current time step and the condition prior. \citet{lu2024procreate} push samples away from a reference set based on approximate denoised values.

\textbf{Diverse Sampling in Diffusion Models using Particles.}
Several approaches achieve diverse sampling in diffusion and flow models through interactions between particles. \citet{kim_diverse_2025} and \citet{harrington2025s} both optimize the initial noise samples: the former with an InfoNCE loss \citep{oord2018representation} on approximate denoised values, the latter with an objective combining quality, diversity and noise regularization. \citet{parmar2026scaling} sub-sample a subset of particles at each flow step using quadratic integer programming that promotes both quality and diversity. 
\citet{kirchhof2025shielded} shift particles whose approximate denoised values fall inside balls centered at the elements of some reference set. More related to our approach are \pg{} \citep{corso_particle_2023}, DiverseFlow \citep{morshed_diverseflow_2025}, and SPARKE \citep{jalali2026sparke}, all of which modify the SDE during sampling using kernels. \pg{} adds a fixed potential function based on similarity kernels, DiverseFlow uses determinantal point processes (DPP) as a repulsion force on the approximate denoised values, and SPARKE, designed to operate on semantically similar prompts simultaneously, adds a repulsion term based on the RKE score \citep{jalali2023information}. Unlike \eddy{}, all of these approaches modify the marginal distribution of particles, hence sacrificing fidelity or quality for diversity. We note that \pg{} discussed the possibility of preserving the marginal distribution of particles by learning the potential function, but they did not follow this path in their experimental evaluations. Lastly, \citet{liu2025importance} use importance weighting to correct the bias introduced by a repulsion term; however, this requires learning an additional model to track the weights, unlike \eddy{} which is training-free. 

\section{The Cost of Divergence-Free Guidance}
\label{sec:theory}
In this section and the next, time $t$ is implicit unless stated otherwise. 
That includes $p = p_t$, $\xp{i} = \xpt{i}$, and derived quantities that depend on $t$. 
The \emph{marginal} $p_t(\xpt{i})$ describes the distribution of a single particle at time $t$.
When sampling a batch of $\np$ particles jointly, their collective behavior is governed by a \emph{joint} distribution $\vp_t(\Xt)$ where $\Xt = (\xpt{1},\ldots,\xpt{\np})$.
We would like a sampler that (i) preserves $p_t$ for every particle individually, and (ii) couples the particles non-trivially and thereby increases diversity.

Claim~\ref{claim:rot_symmetry} is stated for \emph{external} fields: the matrix field $A$ depends solely on the particle's own position, so guided particles never interact and stay independent.
Both claims below concern \emph{interacting} fields, in which $\Ai_t$ sees the whole batch, and together they analyze such guidance.
Claim~\ref{thm:product_propagation} shows it is free but inert on an i.i.d.\ batch: the joint stays a product, so the marginals are exact and no diversity is gained.
Claim~\ref{thm:marginal_cost} shows exactly how the marginal is perturbed when initialized from a non-i.i.d. initialization.

\paragraph{Setup.}
Let $\Xt = (\xpt{i})_{i=1}^{\np}$ be $\np$ particles in $\R^\dd$, $\dd \geq 2$, each individually evolving by drift $\drift_t$ under a probability path $p_t$ which is always nonzero, and let each particle $i$ receive the additional drift $\psii_t(\Xt) = \Apt(\Ai_t(\Xt))$ for an anti-symmetric matrix field $\Ai_t$ which is Lipschitz in $\Xt$.
Write $\vp_t$ for the joint law of $\Xt$ and $\tilde p_t$ for the marginal law of $\xpt{i}$.

\begin{claim}[Independence propagates]
\label{thm:product_propagation}
If $\vp_0 = p_0^{\otimes \np}$ then $\vp_t = p_t^{\otimes \np}$ for every $t \geq 0$.
\end{claim}

\begin{proof}[\textbf{Proof sketch}]
The \emph{joint} density $\vp_t$ on $\R^{\np \dd}$ obeys the FPE
\begin{equation}
    \frac{\partial}{\partial t} \vp_t \;=\; -\sum_{i} \divergence_i\bigl(\vp_t\,(\drift_t + \psii_t)\bigr) \;+\; \frac{\vol_t^2}{2}\sum_{i} \Delta_i \vp_t,
    \label{eq:joint_fpe}
\end{equation}
where $\divergence_i$ and $\Delta_i$ act on the $\xpt{i}$ block. Substitute the product ansatz $\vp_t = p_t^{\otimes \np}$. Each guidance term contributes
\begin{equation}
    \divergence_i\bigl(\vp_t\,\psii_t\bigr) \;=\; \Bigl(\prod_{j \neq i} p_t(\xpt{j})\Bigr)\,\divergence_i \divergence_i \bigl(p_t(\xpt{i})\,\Ai_t(\Xt)\bigr) \;=\; 0,
    \label{eq:rigidity_cancel}
\end{equation}
because the factors $j \neq i$ do not depend on $\xpt{i}$ and pass through $\divergence_i$, and the remaining double divergence vanishes by the mixed-partials cancellation of Claim~\ref{claim:rot_symmetry} applied in the $\xpt{i}$ block. The surviving terms are $\np$ copies of the single-particle FPE~\eqref{eq:fpe}, which $p_t^{\otimes \np}$ satisfies by construction. Note that~\eqref{eq:joint_fpe} is \emph{linear} in $\vp_t$ and admits a unique solution, so the product law is that solution for every $t$.
\end{proof}

Every marginal is therefore exactly $p_t$ and nothing else happens.
The batch is distributed exactly as $\np$ independent samples, which generates no additional diversity.
We can generate diversity from a coupled initialization, and the next theorem says what that costs.

\begin{claim}[The marginal price of coupling]
\label{thm:marginal_cost}
At any time at which $\tilde p_t = p_t$, the marginal picks up the additive term $\divergence(p_t\,\bbias_t)$ per unit time,
\begin{equation}
    \frac{\partial}{\partial t}\tilde p_t \;=\; \frac{\partial}{\partial t} p_t \;+\; \divergence\bigl(p_t\,\bbias_t\bigr),
    \label{eq:marginal_price}
\end{equation}
where the \emph{marginal bias}
\begin{equation}
    \bbias_t(x) \;=\; \E\Bigl[\,\Ai_t\bigl(x, \xpnit{i}\bigr)\,\nabla_x \log \vp_t\bigl(\xpnit{i} \mid \xpt{i} = x\bigr) \;\Big|\; \xpt{i} = x \,\Bigr]
    \label{eq:marginal_bias}
\end{equation}
vanishes identically if the particles are independent at time $t$. 
\end{claim}

\begin{proof}[\textbf{Proof sketch}]
Write $x := \xpt{i}$ and $y := \xpnit{i}$, and integrate~\eqref{eq:joint_fpe} over $y$. Every term with $\divergence_j$ or $\Delta_j$, $j \neq i$, is a divergence in a variable the integral covers and drops out, leaving the marginal equation
\begin{equation}
    \frac{\partial}{\partial t} \tilde p_t \;=\; -\divergence_x\!\left[\drift_t\,\tilde p_t \;+\; \int \vp_t(x,y)\,\psii_t(x,y)\,dy\right] \;+\; \frac{\vol_t^2}{2}\,\Delta_x \tilde p_t.
    \label{eq:marg_fpe_sketch}
\end{equation}
Factoring $\vp_t(x,y) = \tilde p_t(x)\,c_t(y \mid x)$ and moving the conditional inside the divergence via the product rule $\divergence_x(p_t \Ai_t c_t) = c_t\,\divergence_x(p_t \Ai_t) + p_t \Ai_t \nabla_x c_t$ splits the guidance integral in two,
\begin{equation}
    \int \vp_t\,\psii_t\,dy \;=\; \tilde p_t\bigl[\,\Apt(\bar A_t) \;-\; \bbias_t\,\bigr],
    \qquad
    \bar A_t(x) \;:=\; \E\bigl[\,\Ai_t(x, \xpnit{i}) \;\big|\; \xpt{i} = x\,\bigr].
    \label{eq:bias_split}
\end{equation}
Both enter~\eqref{eq:marg_fpe_sketch} inside $\divergence_x$, so at a time where $\tilde p_t = p_t$ what decides whether the marginal moves is their divergence. Claim~\ref{claim:rot_symmetry} kills the first, $\bar A_t$ being an average of anti-symmetric matrices and depending on $x$ alone:
\begin{equation}
    \divergence_x\bigl(p_t\,\Apt(\bar A_t)\bigr) \;=\; \divergence_x \divergence_x\bigl(p_t \bar A_t\bigr) \;=\; 0.
    \label{eq:abar_inert_sketch}
\end{equation}
It kills $\Ai_t(\cdot, y)$ for each \emph{fixed} $y$ too, so the second would vanish as well if $\divergence_x$ could pass inside the $y$-average; it cannot, because that average is weighted by $c_t(y \mid x)$ and the product rule charges for the weight's own $x$-dependence:
\begin{equation}
    \divergence_x\!\left(\int \vp_t\,\psii_t\,dy\right) \;=\; \int \Bigl[\,\underbrace{\divergence_x\bigl(p_t\,\psii_t(\cdot,y)\bigr)}_{=\,0}\,c_t \;+\; p_t\,\psii_t \cdot \nabla_x c_t \,\Bigr]\,dy \;=\; -\,\divergence_x\bigl(p_t\,\bbias_t\bigr).
    \label{eq:bias_survives}
\end{equation}
What survives is a \emph{correlation} between the guidance and the neighbors' conditional score $\nabla_x \log c_t$, which anti-symmetry does not constrain. Substituting it back into~\eqref{eq:marg_fpe_sketch} leaves the marginal evolving with drift $\drift_t - \bbias_t$, which is~\eqref{eq:marginal_price} once the target's own evolution~\eqref{eq:fpe} is subtracted. Independent particles have $c_t(y \mid x) = c_t(y)$, so $\nabla_x \log c_t \equiv 0$ and $\bbias_t \equiv 0$.
\end{proof}
The full proofs are in Appendix~\ref{app:proof}.

\begin{remark}
    Claim~\ref{thm:product_propagation} is a continuous-time statement.
    An explicit Euler sampler may couples even an i.i.d.\ batch by treating time as discrete.
\end{remark}

\section{Method}
\label{sec:method}
\subsection{Constructing \eddyrbf{}}
\label{sec:eddy_rbf}

We first instantiate the principle into a sampling algorithm with negatively correlated priors and RBF kernel. We will, in the next subsection, generalize the construction to arbitrary similarity kernels, but this requires approximations to scale to high dimensions. 

First, we couple the prior. 
Claim~\ref{thm:product_propagation} shows that the drift alone cannot introduce dependence, thus it is required.
We negatively correlate the priors while preserving their marginals by: drawing $\np$ particles independently from $p_0$, subtracting their empirical mean and rescaling by $\sqrt{n/(n-1)}$.
For more details, see Claim~\ref{claim:coupled_prior} (Appendix~\ref{app:coupled_prior}).

Second, we perturb the drift. For each particle $\xp{i}$ and each neighbor $\xp{j} \neq \xp{i}$, we form the anti-symmetric matrix $\Aj$ as follows:
\begin{equation}
    \Aj \;=\; \vj \otimes \rij \;-\; \rij \otimes \vj,
    \label{eq:eddy_aij}
\end{equation}
where $\otimes$ is the outer product of vectors, $\vj$ is a vector that depends on neighbor $j$, and $\rij = - \nabla \kij$ is the repulsive direction induced by a similarity kernel $k$.
Intuitively, $\rij \otimes \vj$ is a rank-one matrix whose range is the desired direction $\rij$, and subtracting its transpose makes it anti-symmetric. 
We let $\vj$ depend on neighbour $j$ to avoid additional costly terms in the divergence and to incorporate information about neighbour $j$.
In our experiments, we found that $\vj=\muj$ (the drift) works well for OT-FM as well as the synthetic experiment, while $\vj = \sigma \,\score(\xp{j})$ works well for SDXL where $\sigma = \sigma_t$ is the current noise in the schedule.

To account for the contribution of all particles on $\xp{i}$ we average the  contribution over all neighbors, which yields the per-particle guidance:
    $\psii_{\eddy{}} \;=\; \frac{1}{\np-1} \sum_{j\neq i} \Ap(\Aj).
    \label{eq:eddy_conceptual}$. Each $\Aj$ is anti-symmetric, so this quantity disturbs every particle's marginal only through the marginal bias~\eqref{eq:marginal_bias}.

By Definition~\ref{def:Ap_matrix}, $\Ap(\Aj) = \divergence \Aj + \Aj\,\si$ where $\si = \score(\xp{i})$.
The score $\si$ is available either directly for diffusion or via Tweedie's formula using Eq. ~\ref{eq:tweedie_flow} for flow matching, reducing the problem to evaluating the divergence:
\begin{equation}
    \divergence \Aj \;=\; (\divergence\rij)\,\vj \;+\; \cancel{(\nabla \vj)\,\rij} \;-\; \cancel{(\divergence \vj)\,\rij} \;-\; (\nabla \rij)\, \vj.
    \label{eq:div_antisym}
\end{equation}
Where, the terms in Eq.~\ref{eq:div_antisym} disappear because $\vj$ does not depend on $\xp{i}$, so its gradient and divergence with respect to $\xp{i}$ are zero.
The remaining terms are pure kernel quantities, $\divergence \rij = -\Delta k(\xp{i}, \xp{j})$ and $(\nabla\rij)\,\vj = -\nabla^2 \kij\,\vj$, both \textbf{do not} require backpropagation through the generative model. These combine into $\divergence \Aj = \Kij\,\vj$, where $\Kij = \nabla^2 \kij - \Delta \kij\,I_{\dd}$ is the divergence-free matrix kernel \citep{narcowich_generalized_1994, ni_representing_2025}, classically used to represent flow and electromagnetic fields.
Eq.~\ref{eq:eddy_conceptual} thus simplifies to:
\begin{equation}
    \psii_{\eddy{}} \;=\; \frac{1}{\np-1}\sum_{j\neq i}\bigl(\Aj\,\si \;+\; \Kij\,\vj\bigr).
    \label{eq:eddy_decomp}
\end{equation}
For the RBF kernel $\kij = \exp\bigl(-\|\xp{i} - \xp{j}\|^2/\gamma\bigr)$ with bandwidth $\gamma$, all the kernel derivatives admit closed forms, resulting in an easy to calculate formula (see Appendix~\ref{app:eddy_rbf_cff}) that we name \eddyrbf.

\subsection{Constructing \eddy{}}
\label{sec:eddy_general}

\eddyrbf{} assumes the kernel derivatives are given in closed form. In our text-to-image experiments (Section~\ref{sec:exp_t2i}) we use an RBF kernel applied to DINOv2 features, so the second-order derivatives in~Eq. \ref{eq:eddy_decomp} 
are no longer known analytically. Moreover, to exactly compute the Laplacian element of \eddy{} would require $\mathcal{O}(d)$ backpropagations, which is impractical for high-dimensional problems.

We therefore drop the closed-form assumption and approximate both second-order quantities using only fast kernel evaluations and first-order kernel gradients (as available from CLIP or DINO). This adds a third source of distortion, on top of the marginal bias of Claim~\ref{thm:marginal_cost}, yet \eddy{} still outperforms competing approaches in our experiments.

For the Hessian--vector product, we use a central difference with two backward passes,
\begin{equation}
    \nabla^2\kij\,\vj \;\approx\; \frac{\nabla k(\xpt{i}+\fdeps\vj,\xpt{j}) - \nabla k(\xpt{i}-\fdeps\vj,\xpt{j})}{2\fdeps}.
    \label{eq:hvp_fd}
\end{equation}
For the Laplacian, we use a forward-only Hutchinson trace estimator with $m = 5$ Rademacher probes $r_1,\ldots,r_m \overset{\mathrm{iid}}{\sim} \mathrm{Rad}(\{\pm 1\}^{\dd})$, requiring $2m+1$ forward passes,
\begin{equation}
    \Delta \kij \;\approx\; \frac{1}{m}\sum_{\ell=1}^{m} \frac{k(\xpt{i}+\fdeps r_\ell,\xpt{j}) - 2\kij + k(\xpt{i}-\fdeps r_\ell,\xpt{j})}{\fdeps^2}.
    \label{eq:lap_hutch}
\end{equation}
Empirically, the guidance has its largest effect in the early, high-noise portion of the trajectory, so we apply $\psii_{\eddy{}}$ only for the first fraction $r_{\textup{stop}} = 0.2$ of sampling steps, cutting the overhead fivefold with negligible loss in output quality.
Algorithm~\ref{alg:eddy} in the appendix summarizes the full procedure, and Table~\ref{tab:elapsed_time} reports average sampling times for \eddy{}, \base{}, \pg{} and \df{} under the same $0.2$ fraction.

\section{Experiments}
\label{sec:experiments}
We first evaluate on a synthetic low-dimensional model where standard statistical tests can resolve how far the marginal distribution has moved, and then on modern text-to-image generation models.

\subsection{2D Gaussian Mixture Model}
\label{sec:exp_2d}

\textbf{Experimental Setting. }
We evaluate \eddyrbf{} on a 2-dimensional data model for which we can compute the process analytically. 
Specifically, we use OT-FM to transport a batch of $5$ particles from the prior $p_0 = \N(0,I_2)$ to target $p_1$ an equal mixture of five Gaussians with unit variance, where the $\ell^\text{th}$ Gaussian is centered at $5\cdot \bigl(\sin (2\pi\ell/5), \cos (2\pi\ell/5)\bigr)$.
We simulate $5{,}000$ batches with $10{,}000$ uniform Euler steps.
Appendix~\ref{app:synthetic} repeats the protocol for a Langevin-corrected SDE sampler.

\begin{figure}[t]
    \centering
    \includegraphics[width=\linewidth]{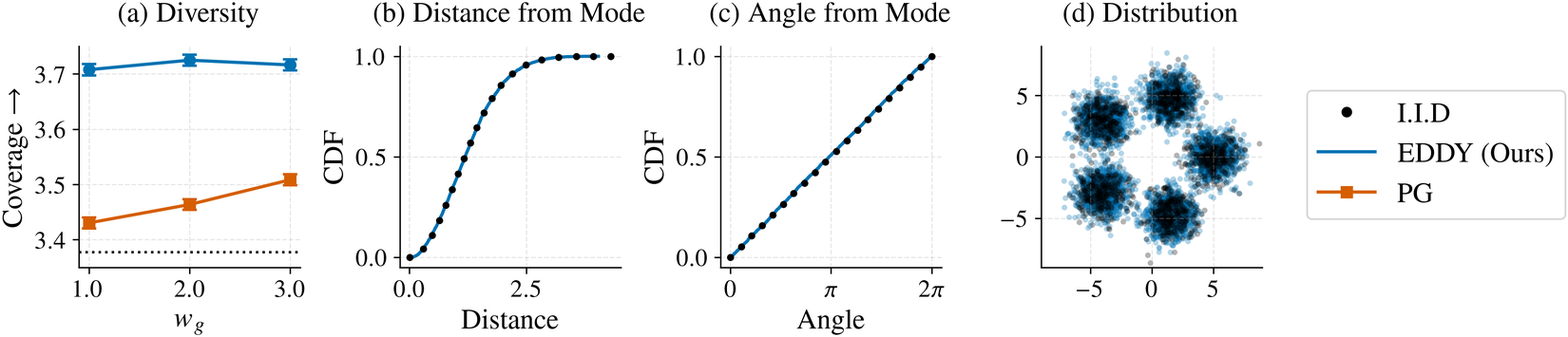}
    \caption{
        (a) Mean mode coverage of \eddy{} and \pg{} vs guidance weight $\wg$. \base{} baseline (dotted line) denoting the expected coverage for independent particles.
        (b) Empirical CDF of nearest-mode distances for \eddy{} ($\wg=2.0$) versus \base{}.
        (c) Empirical CDF of nearest-mode angle for \eddy{} ($\wg=2.0$) versus \base{}.
        (d) Samples from \eddy{} and \base{}.
    }
    \label{fig:gmm}
\end{figure}

\textbf{Evaluation metrics. }
We measure diversity by the expected mode coverage: how many unique modes are a nearest neighbor for some particle. The baseline coverage for independent (non-interacting) particles can be calculated directly as $5\bigl[1-(1-\tfrac{1}{5})^5\bigr] \approx 3.36$.
We test fidelity with two test statistics: the distance of each sample to the nearest Gaussian center and its angle around it. 
We run Kolmogorov--Smirnov, Mann--Whitney and Welch two-sample tests between \eddyrbf{}, \pg{} and \base{} samples, taking one particle per batch to obtain independent marginal samples.

\textbf{Experimental results. }
Figure~\ref{fig:gmm}(a) shows that mode coverage increases with \eddy{}'s guidance coefficient $\wg$ outperforming \pg{}.
At all operating points \eddyrbf{} passes all six tests and its distance and angle CDFs stay indistinguishable from \base{} (Fig.~\ref{fig:gmm}(b,c)).

\subsection{Text-to-Image Generation}
\label{sec:exp_t2i}

\textbf{Experimental Setting. }
We compare \eddy{} against \pg{} \citep{corso_particle_2023} and \df{} \citep{morshed_diverseflow_2025} on both a flow matching model and a latent diffusion model.
We additionally evaluate \cads{} \citep{sadat_cads_2024}, but because it induces diversity by perturbing the conditioning signal rather than by coupling particles, it trades prompt adherence for diversity along a different axis than the methods above; we therefore report it separately in Appendix~\ref{app:cads}.
We also attempted to compare against CNO \citep{kim_diverse_2025}, but official code was unavailable at the time of submission and, despite substantial hyperparameter tuning, our own implementation never reached a configuration that improved diversity while staying competitive in image quality.
We therefore exclude it from the main comparison to avoid drawing potentially misleading conclusions. 
We additionally report that using solely the coupled prior does not change our metrics.

We test on FLUX.1-dev \citep{noauthor_flux1_2024} and FLUX.2-klein-base-9B \citep{blackforestlabs_flux2_2026} for flow matching and Stable Diffusion XL \citep{podell_sdxl_2023} (SDXL) for diffusion on 2048 randomly sampled validation prompts from MS-COCO \citep{lin_microsoft_2014}, generating 4 images per caption. 
During inference, we use \eddy{}, \pg{} and \df{} with a RBF kernel on features extracted by composing a fast latent decoder with small DINOv2 \citep{darcet_vision_2023}, which we found to work better than the CLIP \citep{radford_learning_2021} similarity measure. 
For a fair comparison, we tune each method's hyperparameters to a comparable diversity regime; the chosen values are listed in Table~\ref{tab:chosen_hparams} and the backbone settings in Table~\ref{tab:base_hparam} (Appendix~\ref{app:t2i_results}).
Despite extensive tuning, we could not bring \df{} into a comparable diversity regime on SDXL, and therefore omit it from that backbone.

\textbf{Evaluation metrics. }
We measure diversity by using the DINO pairwise similarity score between images as was done in \citet{corso_particle_2023} with large DINOv2.
Quality can be divided into two components: prompt alignment and image quality.
We measure prompt alignment via the CLIPScore \citep{hessel_clipscore_2021}, and we measure image quality using CMMD \citep{jayasumana_rethinking_2024}, Aesthetic score \citep{LAION-AI}, and FID \citep{heusel_gans_2017}. 
CMMD and FID are computed by independent samples from the marginal distributions, meaning we only use one particle per batch. 
To set a FID score for the baseline samples, we compute FID between two non-overlapping sets of samples from the base I.I.D model.

\begin{figure}[t]
    \centering
    \includegraphics[width=\linewidth]{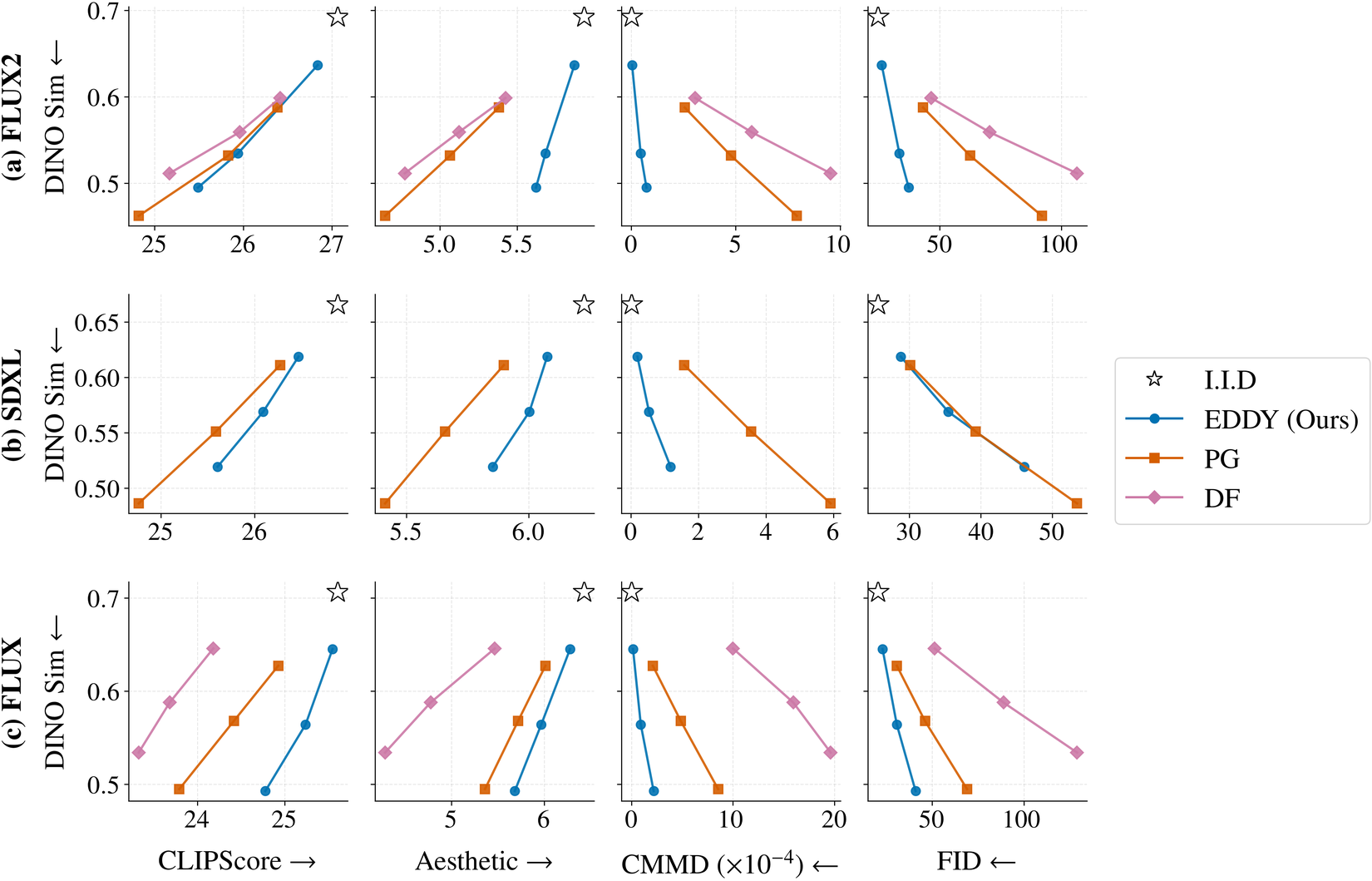}
    \caption{
        Diversity-quality trade-off curves for \eddy{}, Particle Guidance (\pg{}), and DiverseFlow (\df{}) on MS-COCO, for (a) FLUX.2-klein-base-9B, (b) SDXL, and (c) FLUX.1-dev.
        Each point corresponds to a different guidance strength; the white star marks the \base{} baseline.
        The $y$-axis measures pairwise DINO similarity ($\leftarrow$ lower = more diverse).
    }
    \label{fig:pareto}
\end{figure}

\textbf{Experimental results.}
Quantitative results for FLUX.1-dev, SDXL, and FLUX.2-klein-base-9B are presented in Fig. \ref{fig:pareto}, with exact values reported in Tables~\ref{tab:pareto_flux},~\ref{tab:pareto_sdxl}, and~\ref{tab:pareto_flux2} in the appendix. 
As shown, \eddy{} consistently achieves higher image quality than \pg{} and \df{} at matched diversity levels, across all four quality metrics and all three backbones.
Qualitatively, \pg{} and \df{} introduce pronounced visual artifacts, whereas \eddy{} at worst exhibits mild blurring, resulting in less perceptual degradation overall.
Representative samples are provided in the supplementary material.

\section{Limitations}
\label{sec:limitations}
Our approach has several limitations, and the first is theoretical.
\eddy{} does not preserve the marginal exactly, and Claim~\ref{thm:product_propagation} shows no method in its family could; Section~\ref{sec:exp_2d} measures the marginal bias, but for text-to-image generation the marginal is inaccessible.
Second, part of the reported diversity is solver-dependent (Appendix~\ref{app:synthetic}) and need not survive a finer solver.
Third, \eddy{} adds runtime overhead: roughly $14\%$ on FLUX.1-dev, $10\%$ on FLUX.2-klein-base-9B and 6 seconds on SDXL (Table~\ref{tab:elapsed_time}).
Finally, we observe occasional blurriness with FLUX.1-dev, a known issue for that model which may reflect a mode of diversity consistent with the data; we do not observe it on SDXL, whose image quality is lower.

\section{Conclusion}
\label{sec:conclusion}
We asked whether particle guidance can promote diversity without disturbing each sample's own distribution, and proved that it cannot: independently initialized particles are rigid, while coupled particles may disturb the marginal.
The closed form is what makes the negative result constructive: it yields \eddy{}, which improves the diversity--fidelity trade-off at matched diversity.

\subsection*{AI use statement}

In this work, we used generative AI tools for proofreading, copyediting, and improving the phrasing of the manuscript. We have not used generative AI tools for generating primary research ideas, analyzing experimental data, or writing the literature review, and tasks related to generating synthetic datasets are not applicable to this work. Additionally, we used generative AI tools for drafting boilerplate code for data visualization. We have reviewed all AI-assisted work. Specifically, all AI-suggested text edits were manually reviewed by the lead author to ensure original intent was preserved, and all AI-generated code was executed, tested, and verified for accuracy against our raw dataset. We take responsibility for the final content of this work, including text, claims or artifacts produced with the aid of generative AI.

\bibliography{iclr2027_conference}

@misc{LAION-AI,
  author = {{LAION-AI Team}},
  title  = {Laion-aesthetics predictor v2},
  year   = {2022},
  url    = {https://github.com/christophschuhmann/improved-aesthetic-predictor}
}

@inproceedings{
podell_sdxl_2023,
title={{SDXL}: Improving Latent Diffusion Models for High-Resolution Image Synthesis},
author={Dustin Podell and Zion English and Kyle Lacey and Andreas Blattmann and Tim Dockhorn and Jonas M{\"u}ller and Joe Penna and Robin Rombach},
booktitle={The Twelfth International Conference on Learning Representations},
year={2024},
url={https://openreview.net/forum?id=di52zR8xgf}
}

@article{oord2018representation,
  title={Representation learning with contrastive predictive coding},
  author={Oord, Aaron van den and Li, Yazhe and Vinyals, Oriol},
  journal={arXiv preprint arXiv:1807.03748},
  year={2018}
}

@article{narcowich_generalized_1994,
	title = {Generalized {Hermite} interpolation via matrix-valued conditionally positive definite functions},
	volume = {63},
	issn = {0025-5718},
	url = {https://www.ams.org/mcom/1994-63-208/S0025-5718-1994-1254147-6/},
	doi = {10.1090/S0025-5718-1994-1254147-6},
	language = {en},
	number = {208},
	urldate = {2026-04-21},
	journal = {Mathematics of Computation},
	author = {Narcowich, Francis J. and Ward, Joseph D.},
	month = jan,
	year = {1994},
	pages = {661--661},
}

@inproceedings{
corso_particle_2023,
title={Particle Guidance: non-{I.I.D.} Diverse Sampling with Diffusion Models},
author={Gabriele Corso and Yilun Xu and Valentin De Bortoli and Regina Barzilay and Tommi S. Jaakkola},
booktitle={The Twelfth International Conference on Learning Representations},
year={2024},
url={https://openreview.net/forum?id=KqbCvIFBY7}
}

@inproceedings{
kim_diverse_2025,
title={Diverse Text-to-Image Generation via Contrastive Noise Optimization},
author={Byungjun Kim and Soobin Um and Jong Chul Ye},
booktitle={The Fourteenth International Conference on Learning Representations},
year={2026},
url={https://openreview.net/forum?id=EVRMnAREc3}
}

@inproceedings{
lipman_flow_2022,
title={Flow Matching for Generative Modeling},
author={Yaron Lipman and Ricky T. Q. Chen and Heli Ben-Hamu and Maximilian Nickel and Matthew Le},
booktitle={The Eleventh International Conference on Learning Representations },
year={2023},
url={https://openreview.net/forum?id=PqvMRDCJT9t}
}

@inproceedings{
sadat_cads_2024,
title={{CADS}: Unleashing the Diversity of Diffusion Models through Condition-Annealed Sampling},
author={Seyedmorteza Sadat and Jakob Buhmann and Derek Bradley and Otmar Hilliges and Romann M. Weber},
booktitle={The Twelfth International Conference on Learning Representations},
year={2024},
url={https://openreview.net/forum?id=zMoNrajk2X}
}

@inproceedings{morshed_diverseflow_2025,
	title = {{DiverseFlow}: {Sample}-efficient diverse mode coverage in flows},
	booktitle = {Proceedings of the {Computer} {Vision} and {Pattern} {Recognition} {Conference}},
	author = {Morshed, Mashrur M and Boddeti, Vishnu},
	year = {2025},
	pages = {23303--23312},
}

@article{ho_denoising_2020,
	title = {Denoising diffusion probabilistic models},
	volume = {33},
	journal = {Advances in Neural Information Processing Systems},
	author = {Ho, Jonathan and Jain, Ajay and Abbeel, Pieter},
	year = {2020},
	pages = {6840--6851},
}

@inproceedings{song_score-based_2021,
	title = {Score-{Based} {Generative} {Modeling} through {Stochastic} {Differential} {Equations}},
	booktitle = {9th {International} {Conference} on {Learning} {Representations} {ICLR}},
	publisher = {OpenReview.net},
	author = {Song, Yang and Sohl-Dickstein, Jascha and Kingma, Diederik P. and Kumar, Abhishek and Ermon, Stefano and Poole, Ben},
	year = {2021},
}

@article{ni_representing_2025,
  title={Representing Flow Fields with Divergence-Free Kernels for Reconstruction},
  author={Ni, Xingyu and Xing, Jingrui and Li, Xingqiao and Wang, Bin and Chen, Baoquan},
  journal={Proceedings of the ACM on Computer Graphics and Interactive Techniques},
  volume={8},
  number={4},
  pages={1--21},
  year={2025},
  publisher={ACM New York, NY}
}

@inproceedings{
liu_flow_2022,
title={Flow Straight and Fast: Learning to Generate and Transfer Data with Rectified Flow},
author={Xingchao Liu and Chengyue Gong and qiang liu},
booktitle={The Eleventh International Conference on Learning Representations },
year={2023},
url={https://openreview.net/forum?id=XVjTT1nw5z}
}

@inproceedings{jayasumana_rethinking_2024,
  title={Rethinking {FID}: Towards a better evaluation metric for image generation},
  author={Jayasumana, Sadeep and Ramalingam, Srikumar and Veit, Andreas and Glasner, Daniel and Chakrabarti, Ayan and Kumar, Sanjiv},
  booktitle={Proceedings of the IEEE/CVF conference on computer vision and pattern recognition},
  pages={9307--9315},
  year={2024}
}

@inproceedings{
ho_classifier-free_2021,
title={Classifier-Free Diffusion Guidance},
author={Jonathan Ho and Tim Salimans},
booktitle={NeurIPS 2021 Workshop on Deep Generative Models and Downstream Applications},
year={2021},
url={https://openreview.net/forum?id=qw8AKxfYbI}
}

@inproceedings{radford_learning_2021,
  title={Learning transferable visual models from natural language supervision},
  author={Radford, Alec and Kim, Jong Wook and Hallacy, Chris and Ramesh, Aditya and Goh, Gabriel and Agarwal, Sandhini and Sastry, Girish and Askell, Amanda and Mishkin, Pamela and Clark, Jack and others},
  booktitle={International conference on machine learning},
  pages={8748--8763},
  year={2021},
  organization={PmLR}
}

@inproceedings{
darcet_vision_2023,
title={Vision Transformers Need Registers},
author={Timoth{\'e}e Darcet and Maxime Oquab and Julien Mairal and Piotr Bojanowski},
booktitle={The Twelfth International Conference on Learning Representations},
year={2024},
url={https://openreview.net/forum?id=2dnO3LLiJ1}
}

@inproceedings{lin_microsoft_2014,
  title={Microsoft {COCO}: Common objects in context},
  author={Lin, Tsung-Yi and Maire, Michael and Belongie, Serge and Hays, James and Perona, Pietro and Ramanan, Deva and Doll{\'a}r, Piotr and Zitnick, C Lawrence},
  booktitle={European conference on computer vision},
  pages={740--755},
  year={2014},
  organization={Springer}
}

@inproceedings{heusel_gans_2017,
	title = {{GANs} {Trained} by a {Two} {Time}-{Scale} {Update} {Rule} {Converge} to a {Local} {Nash} {Equilibrium}},
	volume = {30},
	url = {https://proceedings.neurips.cc/paper/2017/hash/8a1d694707eb0fefe65871369074926d-Abstract.html},
	urldate = {2026-04-21},
	booktitle = {Advances in {Neural} {Information} {Processing} {Systems}},
	publisher = {Curran Associates, Inc.},
	author = {Heusel, Martin and Ramsauer, Hubert and Unterthiner, Thomas and Nessler, Bernhard and Hochreiter, Sepp},
	year = {2017},
}

@inproceedings{hessel_clipscore_2021,
	address = {Online and Punta Cana, Dominican Republic},
	title = {{CLIPScore}: {A} {Reference}-free {Evaluation} {Metric} for {Image} {Captioning}},
	shorttitle = {{CLIPScore}},
	url = {https://aclanthology.org/2021.emnlp-main.595/},
	doi = {10.18653/v1/2021.emnlp-main.595},
	urldate = {2026-04-21},
	booktitle = {Proceedings of the 2021 {Conference} on {Empirical} {Methods} in {Natural} {Language} {Processing}},
	publisher = {Association for Computational Linguistics},
	author = {Hessel, Jack and Holtzman, Ari and Forbes, Maxwell and Le Bras, Ronan and Choi, Yejin},
	editor = {Moens, Marie-Francine and Huang, Xuanjing and Specia, Lucia and Yih, Scott Wen-tau},
	month = nov,
	year = {2021},
	pages = {7514--7528},
}

@article{efron_tweedies_2011,
	title = {Tweedie's {Formula} and {Selection} {Bias}},
	volume = {106},
	issn = {0162-1459},
	url = {https://www.jstor.org/stable/23239562},
	number = {496},
	urldate = {2026-04-21},
	journal = {Journal of the American Statistical Association},
	publisher = {[American Statistical Association, Taylor \& Francis, Ltd.]},
	author = {Efron, Bradley},
	year = {2011},
	pages = {1602--1614},
}

@misc{noauthor_flux1_2024,
	title = {{FLUX}.1 [dev]},
	url = {https://huggingface.co/black-forest-labs/FLUX.1-dev},
	author = {{Black Forest Labs}},
	publisher = {Black Forest Labs},
	year = {2024},
}

@misc{blackforestlabs_flux2_2026,
	title = {{FLUX}.2},
	url = {https://huggingface.co/black-forest-labs/FLUX.2-klein-base-9B},
	author = {{Black Forest Labs}},
	publisher = {Black Forest Labs},
	year = {2026},
}

@article{anderson_reverse-time_1982,
	title = {Reverse-time diffusion equation models},
	volume = {12},
	issn = {0304-4149},
	url = {https://www.sciencedirect.com/science/article/pii/0304414982900515},
	doi = {10.1016/0304-4149(82)90051-5},
	number = {3},
	urldate = {2026-04-24},
	journal = {Stochastic Processes and their Applications},
	author = {Anderson, Brian D. O.},
	month = may,
	year = {1982},
	pages = {313--326},
}

@article{richter-powell_neural_2022,
	title = {Neural {Conservation} {Laws}: {A} {Divergence}-{Free} {Perspective}},
	volume = {35},
	shorttitle = {Neural {Conservation} {Laws}},
	url = {https://proceedings.neurips.cc/paper_files/paper/2022/hash/f8d39584f87944e5dbe46ec76f19e20a-Abstract-Conference.html},
	language = {en},
	urldate = {2026-04-29},
	journal = {Advances in Neural Information Processing Systems},
	author = {Richter-Powell, Jack and Lipman, Yaron and Chen, Ricky T. Q.},
	month = dec,
	year = {2022},
	pages = {38075--38088},
}

@inproceedings{hua_simulation-free_2025,
  title={Simulation-Free Differential Dynamics Through Neural Conservation Laws},
  author={Hua, Mengjian and Vanden-Eijnden, Eric and Chen, Ricky TQ},
  booktitle={Conference on Uncertainty in Artificial Intelligence},
  pages={1730--1744},
  year={2025},
  organization={PMLR}
}

@incollection{stein_bound_1972,
	title = {A bound for the error in the normal approximation to the distribution of a sum of dependent random variables},
	volume = {6.2},
	language = {en},
	urldate = {2026-04-29},
	booktitle = {Proceedings of the {Sixth} {Berkeley} {Symposium} on {Mathematical} {Statistics} and {Probability}, {Volume} 2: {Probability} {Theory}},
	publisher = {University of California Press},
	author = {Stein, Charles},
	month = jan,
	year = {1972},
	pages = {583--603},
}

@inproceedings{
parmar2026scaling,
title={Scaling Group Inference for Diverse and High-Quality Generation},
author={Gaurav Parmar and Or Patashnik and Daniil Ostashev and Kuan-Chieh Jackson Wang and Kfir Aberman and Srinivasa Narasimhan and Jun-Yan Zhu},
booktitle={The Fourteenth International Conference on Learning Representations},
year={2026},
url={https://openreview.net/forum?id=IyTNxjTuWT}
}

@article{harrington2025s,
  title={It's Never Too Late: Noise Optimization for Collapse Recovery in Trained Diffusion Models},
  author={Harrington, Anne and Koepke, A and Karthik, Shyamgopal and Darrell, Trevor and Efros, Alexei A},
  journal={arXiv preprint arXiv:2601.00090},
  year={2025}
}

@inproceedings{kirchhof2025shielded,
  title={Shielded Diffusion: Generating Novel and Diverse Images using Sparse Repellency},
  author={Kirchhof, Michael and Thornton, James and B{\'e}thune, Louis and Ablin, Pierre and Ndiaye, Eugene and Cuturi, Marco},
  booktitle={International Conference on Machine Learning},
  pages={30911--30942},
  year={2025},
  organization={PMLR}
}

@inproceedings{
jalali2026sparke,
title={{SPARKE}: Scalable Prompt-Aware Diversity and Novelty Guidance in Diffusion Models via {RKE} Score},
author={Mohammad Jalali and Haoyu LEI and Amin Gohari and Farzan Farnia},
booktitle={The Thirty-ninth Annual Conference on Neural Information Processing Systems},
year={2026},
url={https://openreview.net/forum?id=1YLpf8nUIq}
}

@article{liu2025importance,
  title={Importance-Weighted Non-{IID} Sampling for Flow Matching Models},
  author={Liu, Xinshuang and Li, Runfa Blark and Wei, Shaoxiu and Nguyen, Truong},
  journal={arXiv preprint arXiv:2511.17812},
  year={2025}
}

@article{jalali2023information,
  title={An information-theoretic evaluation of generative models in learning multi-modal distributions},
  author={Jalali, Mohammad and Li, Cheuk Ting and Farnia, Farzan},
  journal={Advances in Neural Information Processing Systems},
  volume={36},
  pages={9931--9943},
  year={2023}
}

@article{dhariwal2021diffusion,
  title={Diffusion models beat {GAN}s on image synthesis},
  author={Dhariwal, Prafulla and Nichol, Alexander},
  journal={Advances in neural information processing systems},
  volume={34},
  pages={8780--8794},
  year={2021}
}

@article{kynkaanniemi2024applying,
  title={Applying guidance in a limited interval improves sample and distribution quality in diffusion models},
  author={Kynk{\"a}{\"a}nniemi, Tuomas and Aittala, Miika and Karras, Tero and Laine, Samuli and Aila, Timo and Lehtinen, Jaakko},
  journal={Advances in Neural Information Processing Systems},
  volume={37},
  pages={122458--122483},
  year={2024}
}

@inproceedings{
koulischer2026feedback,
title={Feedback Guidance of Diffusion Models},
author={Felix Koulischer and Florian Handke and Johannes Deleu and Thomas Demeester and Luca Ambrogioni},
booktitle={The Thirty-ninth Annual Conference on Neural Information Processing Systems},
year={2026},
url={https://openreview.net/forum?id=8ySOcf7UpM}
}

@inproceedings{lu2024procreate,
  title={Pro{C}reate, don’t reproduce! propulsive energy diffusion for creative generation},
  author={Lu, Jack and Teehan, Ryan and Ren, Mengye},
  booktitle={European Conference on Computer Vision},
  pages={397--414},
  year={2024},
  organization={Springer}
}

@inproceedings{sohl2015deep,
  title={Deep unsupervised learning using nonequilibrium thermodynamics},
  author={Sohl-Dickstein, Jascha and Weiss, Eric and Maheswaranathan, Niru and Ganguli, Surya},
  booktitle={International conference on machine learning},
  pages={2256--2265},
  year={2015},
  organization={pmlr}
}

@inproceedings{
albergo2023building,
title={Building Normalizing Flows with Stochastic Interpolants},
author={Michael Samuel Albergo and Eric Vanden-Eijnden},
booktitle={The Eleventh International Conference on Learning Representations },
year={2023},
url={https://openreview.net/forum?id=li7qeBbCR1t}
}

@inproceedings{
song2021denoising,
title={Denoising Diffusion Implicit Models},
author={Jiaming Song and Chenlin Meng and Stefano Ermon},
booktitle={International Conference on Learning Representations},
year={2021},
url={https://openreview.net/forum?id=St1giarCHLP}
}

@inproceedings{rombach2022high,
  title={High-resolution image synthesis with latent diffusion models},
  author={Rombach, Robin and Blattmann, Andreas and Lorenz, Dominik and Esser, Patrick and Ommer, Bj{\"o}rn},
  booktitle={Proceedings of the IEEE/CVF conference on computer vision and pattern recognition},
  pages={10684--10695},
  year={2022}
}

@inproceedings{
xu2022geodiff,
title={Geo{D}iff: A Geometric Diffusion Model for Molecular Conformation Generation},
author={Minkai Xu and Lantao Yu and Yang Song and Chence Shi and Stefano Ermon and Jian Tang},
booktitle={International Conference on Learning Representations},
year={2022},
url={https://openreview.net/forum?id=PzcvxEMzvQC}
}

@inproceedings{hoogeboom2022equivariant,
  title={Equivariant diffusion for molecule generation in 3{D}},
  author={Hoogeboom, Emiel and Satorras, V{\i}ctor Garcia and Vignac, Cl{\'e}ment and Welling, Max},
  booktitle={International conference on machine learning},
  pages={8867--8887},
  year={2022},
  organization={PMLR}
}

@inproceedings{
corso2023diffdock,
title={Diff{D}ock: Diffusion Steps, Twists, and Turns for Molecular Docking},
author={Gabriele Corso and Hannes St{\"a}rk and Bowen Jing and Regina Barzilay and Tommi S. Jaakkola},
booktitle={The Eleventh International Conference on Learning Representations },
year={2023},
url={https://openreview.net/forum?id=kKF8_K-mBbS}
}

@article{jing2022torsional,
  title={Torsional diffusion for molecular conformer generation},
  author={Jing, Bowen and Corso, Gabriele and Chang, Jeffrey and Barzilay, Regina and Jaakkola, Tommi},
  journal={Advances in neural information processing systems},
  volume={35},
  pages={24240--24253},
  year={2022}
}

@article{li2022diffusion,
  title={Diffusion-{LM} improves controllable text generation},
  author={Li, Xiang and Thickstun, John and Gulrajani, Ishaan and Liang, Percy S and Hashimoto, Tatsunori B},
  journal={Advances in neural information processing systems},
  volume={35},
  pages={4328--4343},
  year={2022}
}

@article{austin2021structured,
  title={Structured denoising diffusion models in discrete state-spaces},
  author={Austin, Jacob and Johnson, Daniel D and Ho, Jonathan and Tarlow, Daniel and Van Den Berg, Rianne},
  journal={Advances in neural information processing systems},
  volume={34},
  pages={17981--17993},
  year={2021}
}

@inproceedings{
lou2024discrete,
title={Discrete Diffusion Modeling by Estimating the Ratios of the Data Distribution},
author={Aaron Lou and Chenlin Meng and Stefano Ermon},
booktitle={Forty-first International Conference on Machine Learning},
year={2024},
url={https://openreview.net/forum?id=CNicRIVIPA}
}

@article{sahoo2024simple,
  title={Simple and effective masked diffusion language models},
  author={Sahoo, Subham S and Arriola, Marianne and Schiff, Yair and Gokaslan, Aaron and Marroquin, Edgar and Chiu, Justin T and Rush, Alexander and Kuleshov, Volodymyr},
  journal={Advances in Neural Information Processing Systems},
  volume={37},
  pages={130136--130184},
  year={2024}
}
\bibliographystyle{iclr2027_conference}

\newpage
\appendix
\begin{center}
{\bf \Large{Appendix of ``Diverse Sampling in Diffusion Models with Divergence-Free Particle Guidance''}}
\end{center}

\DoToC

\renewcommand{\thefigure}{A.\arabic{figure}}
\renewcommand{\thetable}{A.\arabic{table}}
\setcounter{figure}{0}
\setcounter{table}{0}

\clearpage
\section{\eddy{} Theory}

\subsection{Proofs of Claims~\ref{thm:product_propagation} and~\ref{thm:marginal_cost}}
\label{app:proof}
\subsection*{Setup}
Consider the $\np$-particle system of Section~\ref{sec:theory},
\begin{equation}\label{eq:rig_dyn}
  dx^i_t
  = \bigl[\mu_t(x^i_t) + \psi^i_t(x^i_t, x^{-i}_t)\bigr]\,dt
    + \sigma_t\,dW^i_t,
  \qquad i = 1,\ldots,\np,
\end{equation}
with independent Brownian motions $W^1,\ldots,W^{\np}$ and
$\psi^i_t(\Xt) := \mathcal{A}_{p_t}\bigl(A^{(i)}_t(\Xt)\bigr)$, the
Stein operator taken with respect to the single-particle marginal
$p_t$.
We assume, as in Section~\ref{sec:theory}:
\begin{enumerate}
  \item each $A^{(i)}_t(\Xt)$ is a Lipschitz anti-symmetric matrix field;
  \item $\mu_t$ and $\psi^i_t$ are Lipschitz in all arguments,
    uniformly in $t$, and have at most linear growth;
  \item $\sigma_t$ is bounded, and $p_t > 0$ everywhere.
  \item for each $t > 0$ the joint density is sub-Gaussian together with its first two derivatives: $\vp_t + \|\nabla\vp_t\| + \|\nabla^2\vp_t\| \le C_t\,e^{-c_t\|\Xt\|^2}$ for some $C_t, c_t > 0$.
\end{enumerate}
Assumption~(2) and (4) allow us to assume the flux in both claims vanish at infinity.
Write $\vp_t$ for the joint density of $\Xt = (x^1_t,\ldots,x^{\np}_t)$
on $\R^{\np \dd}$, and let $\divergence_i$, $\nabla_i$, $\Delta_i$
denote the divergence, gradient and Laplacian acting on the $x^i$
block only.
Both claims of Section~\ref{sec:theory} share these assumptions, and the proof
of Claim~\ref{thm:marginal_cost} below uses them unchanged.

\subsection*{Proof of Claim~\ref{thm:product_propagation}}

\begin{proof}
\noindent\textbf{Step 1: The joint Fokker--Planck equation.}
The tuple $\Xt$ evolves under a single SDE on $\R^{\np\dd}$, with velocity field $V_t(\Xt) = \bigl(\drift_t(\xpt{i}) + \psii_t(\xpt{i}, \xpnit{i})\bigr)_{i=1}^{\np}$ and Brownian motion $W_t = (W^1_t,\ldots,W^{\np}_t)$,
\begin{equation}
    d\Xt \;=\; V_t(\Xt)\,dt \;+\; \sigma_t\,dW_t,
\end{equation}
so its Fokker--Planck equation is
\begin{equation}
    \frac{\partial}{\partial t} \vp_t
    \;=\;
    -\divergence(\vp_t V_t) \;+\; \frac{\sigma_t^2}{2} \Delta \vp_t.
\end{equation}

\noindent\textbf{Step 2: The product law is a solution.}
Substitute the ansatz $\vp_t = p_t^{\otimes\np}$, with $p_t$ the solution of the single-particle equation~\plaineqref{eq:fpe}. Block by block,
\begin{equation}
    \divergence_i\bigl(\vp_t \, V_t \bigr)
    \;=\;
    \prod_{j \neq i} p_t(\xpt{j}) \,\cdot\, \left(
        \divergence_i (p_t \drift_t) \,+\, \cancel{\divergence_i (p_t \psii_t)}
    \right),
\end{equation}
because the factors $j \neq i$ do not depend on $\xpt{i}$ and pass through $\divergence_i$, while $p_t \psii_t = \divergence_i(p_t \Ai_t)$ makes the second term a double divergence, which vanishes by the mixed-partials cancellation of Claim~\ref{claim:rot_symmetry} applied in the $\xpt{i}$ block.
Summing over $i$ and using the linearity of the Laplacian leaves $\np$ copies of the single-particle equation, which $p_t$ satisfies by construction.

\noindent\textbf{Step 3: Conclusion.}
The joint equation inherits its regularity from assumptions (1)--(3), so it admits a unique solution --- and the product law is that solution for every $t$.
\end{proof}

The marginals are preserved --- and nothing else happens: the batch is
distributed exactly as $\np$ independent samples, which is the sampler we were
trying to improve on.
This holds for every admissible field, so no choice of $A^{(i)}_t$ buys diversity
from an i.i.d.\ start; only the initialization can.

Claim~\ref{thm:product_propagation} leaves the coupled case open, and this is
where Claim~\ref{thm:marginal_cost} takes over: we keep the
system~\eqref{eq:rig_dyn} and the assumptions listed with it, drop the product
initialization, and locate exactly what obstructs $\tilde p_t = p_t$.

\subsection*{Marginal notation}
Write $x := x^i$ and $y := x^{-i}$ for the $i$-th particle and its neighbors,
$\vp_t(x,y)$ for the joint density under the modified dynamics and
$\tilde p_t(x)$ for the marginal of particle~$i$, and factor
$\vp_t(x,y) = \tilde p_t(x)\,c_t(y \mid x)$ with $c_t$ the conditional law of
the neighbors.

\subsection*{Proof of Claim~\ref{thm:marginal_cost}}

\paragraph{Step 1: Marginal evolution and the bias term.}
The joint density obeys the Fokker--Planck
equation~\plaineqref{eq:joint_fpe}, which in the notation of this appendix reads
\begin{equation}\label{eq:joint_fpe_xy}
  \partial_t \vp_t
  \;=\; -\sum_{j=1}^{\np}\divergence_j\bigl(\vp_t\,(\mu_t + \psi^j_t)\bigr)
    \;+\; \frac{\sigma_t^2}{2}\sum_{j=1}^{\np}\Delta_j \vp_t ,
\end{equation}
and we marginalize it over $y$, that is, integrate it over every block $x^j$
with $j \neq i$.
Every such block drops out: writing $\Delta_j = \divergence_j \nabla_j$, both of
its terms are divergences in a variable the integral covers, so
\begin{align*}
  \int \divergence_j\bigl(\vp_t\,(\mu_t + \psi^j_t)\bigr)\,dy
  &\;=\; -\int \vp_t\,(\mu_t + \psi^j_t)\cdot\nabla_j 1\,dy \;=\; 0 ,\\
  \int \Delta_j \vp_t\,dy
  &\;=\; \int \divergence_j\bigl(\nabla_j \vp_t\bigr)\,dy
   \;=\; -\int \nabla_j \vp_t \cdot \nabla_j 1\,dy \;=\; 0 .
\end{align*}
The $j = i$ terms survive instead, $\divergence_i$ and $\Delta_i$ acting on $x$,
which the $y$-integral treats as a parameter.
The base drift factors out of the one integral that is left, leaving
\begin{equation}\label{eq:marg_fpe}
  \partial_t \tilde p_t
  \;=\; -\divergence_x\!\left[\mu_t\,\tilde p_t
        \;+\; \int \vp_t(x,y)\,\psi^i_t(x,y)\,dy\right]
    \;+\; \tfrac{\sigma_t^2}{2}\,\Delta_x \tilde p_t.
\end{equation}

We now extract the marginal bias from that remaining integral.
Substitute $\psi^i_t = \tfrac{1}{p_t}\divergence_x(p_t A^{(i)}_t)$ and
$\vp_t = \tilde p_t\,c_t$, and write
\begin{equation}\label{eq:abar}
  \bar A_t(x)
  \;:=\; \int A^{(i)}_t(x,y)\,c_t(y \mid x)\,dy
  \;=\; \E\bigl[A^{(i)}_t(x, x^{-i}_t) \mid x^i_t = x\bigr]
\end{equation}
for the conditionally averaged field.
Then
\begin{align}
  \int \vp_t\,\psi^i_t\,dy
  &\;=\; \frac{\tilde p_t}{p_t}\int c_t\,\divergence_x\!\bigl(p_t A^{(i)}_t\bigr)dy
  \notag\\
  &\;=\; \frac{\tilde p_t}{p_t}\left[
      \int \divergence_x\!\bigl(p_t A^{(i)}_t c_t\bigr)dy
      \;-\; \int p_t A^{(i)}_t\,\nabla_x c_t\,dy\right]
  \notag\\
  &\;=\; \frac{\tilde p_t}{p_t}\left[
      \divergence_x\!\bigl(p_t \bar A_t\bigr)
      \;-\; p_t\!\int\! A^{(i)}_t\,\nabla_x \log c_t\;c_t\,dy\right]
  \notag\\
  &\;=\; \tilde p_t\,\bigl[\,\Apt(\bar A_t) \;-\; \bbias_t\,\bigr].
  \label{eq:bias_extract}
\end{align}
Line by line: the first pulls $\tilde p_t / p_t$ out of the integral, neither
factor depending on~$y$; the second applies the product rule
$\divergence_x(p_t A^{(i)}_t c_t)
 = c_t\,\divergence_x(p_t A^{(i)}_t) + p_t A^{(i)}_t \nabla_x c_t$;
the third moves $\divergence_x$ outside the $y$-integral, legal because the
domain of integration does not depend on~$x$, and substitutes
$\nabla_x c_t = c_t \nabla_x \log c_t$; the fourth reads off
$\Apt(\bar A_t) = \tfrac{1}{p_t}\divergence_x(p_t \bar A_t)$ and the marginal bias
$\bbias_t$ of~\eqref{eq:marginal_bias}.

\paragraph{Step 2: Only the coupling term survives.}
Both terms of~\eqref{eq:bias_extract} enter the marginal
equation~\plaineqref{eq:marg_fpe} inside $\divergence_x[\,\cdot\,]$, so what
decides whether the marginal moves is not the terms themselves but their
divergence.
Take a time at which $\tilde p_t = p_t$, so that $\vp_t = p_t\,c_t$, and take
the two terms in turn.

The first term is inert.
$\bar A_t$ is an average of anti-symmetric matrices, hence anti-symmetric, and
it depends on $x$ alone, so Claim~\ref{claim:rot_symmetry} applies to it
verbatim:
\begin{equation}\label{eq:abar_inert}
  \divergence_x\!\bigl(p_t\,\Apt(\bar A_t)\bigr)
  \;=\; \divergence_x \divergence_x\!\bigl(p_t \bar A_t\bigr)
  \;=\; 0 .
\end{equation}

The second term is not, and this is where the price is charged.
Claim~\ref{claim:rot_symmetry} applies just as well to $A^{(i)}_t(\cdot,y)$ for
each \emph{fixed}~$y$, which is anti-symmetric in $x$, so the guidance is
pointwise inert:
\begin{equation}\label{eq:pointwise_inert}
  \divergence_x\!\bigl(p_t\,\psi^i_t(\cdot,y)\bigr)
  \;=\; \divergence_x \divergence_x\!\bigl(p_t A^{(i)}_t(\cdot,y)\bigr)
  \;=\; 0
  \qquad \text{for every fixed } y .
\end{equation}
The guidance would therefore leave the marginal alone if $\divergence_x$ could
be taken inside the $y$-average.
It cannot: that average is weighted by $c_t(y \mid x)$, and the product rule
charges for the weight's own $x$-dependence,
\begin{align}
  \divergence_x\!\left(\int \vp_t\,\psi^i_t\,dy\right)
  &\;=\; \int \divergence_x\!\bigl(p_t\,\psi^i_t\,c_t\bigr)\,dy
  \notag\\
  &\;=\; \int \Bigl[\,
      \underbrace{\divergence_x\!\bigl(p_t\,\psi^i_t\bigr)}_{=\,0}\,c_t
      \;+\; p_t\,\psi^i_t \cdot \nabla_x c_t \,\Bigr]\,dy
  \notag\\
  &\;=\; \int p_t\,\psi^i_t \cdot \nabla_x c_t\,dy .
  \label{eq:heart}
\end{align}
Comparing~\eqref{eq:heart} with~\eqref{eq:bias_extract}
and~\eqref{eq:abar_inert} names what is left:
\begin{equation}\label{eq:bias_div}
  \divergence_x\!\bigl(p_t\,\bbias_t\bigr)
  \;=\; -\,p_t\,\E\bigl[\,\psi^i_t \cdot
        \nabla_x \log c_t(y \mid x) \;\big|\; x^i_t = x \,\bigr] .
\end{equation}
Substituting~\eqref{eq:bias_extract} into~\eqref{eq:marg_fpe}, with
\eqref{eq:abar_inert} removing the averaged field, gives
\begin{equation}\label{eq:marg_drift}
  \partial_t \tilde p_t
  \;=\; -\divergence\bigl(p_t(\mu_t - \bbias_t)\bigr)
        + \tfrac{\sigma_t^2}{2}\Delta p_t
  \;=\; \partial_t p_t \;+\; \divergence\bigl(p_t\,\bbias_t\bigr),
\end{equation}
whose right-hand equality subtracts the target's own
evolution~\plaineqref{eq:fpe} and is the additive
term~\plaineqref{eq:marginal_price} of Claim~\ref{thm:marginal_cost}.
The identity is exact and holds at the instant $t$ alone: the marginal agrees
with $p_t$ there and is displaced from it at rate $\divergence(p_t \bbias_t)$.

\paragraph{Step 3: The bias is charged on the coupling alone.}
Every factor of $\bbias_t$ that is not the field itself is the neighbors'
conditional score $\nabla_x \log c_t(y \mid x)$, which measures how much the
neighbors' law moves when particle $i$ moves.
If the particles are independent at time $t$ then $c_t(y \mid x) = c_t(y)$, so
that score is identically zero and $\bbias_t \equiv 0$: the displacement
in~\eqref{eq:marg_drift} vanishes and the marginal is exactly $p_t$.
Claim~\ref{thm:product_propagation} is the special case in which this holds at
every time at once.
\qed

\subsection{\eddy{} Algorithm}
\subsubsection{Coupled Prior}
\label{app:coupled_prior}

We need an initialization whose joint is not a product, as Claim~\ref{thm:product_propagation} rules out everything else, but whose every marginal is still exactly $p_0 = \N(0, I_\dd)$, so that the marginal starts on the exact side of Claim~\ref{thm:marginal_cost} and departs only by accumulating the bias~\eqref{eq:marginal_bias}. The following construction draws independent particles and then pushes them apart by removing their common component.

\begin{claim}[Coupled prior]
\label{claim:coupled_prior}
Let $\xpt[0]{1},\ldots,\xpt[0]{\np} \overset{\mathrm{iid}}{\sim} \N(0, I_\dd)$ and define
\begin{equation}
    \xap{i} \;=\; \sqrt{\frac{\np}{\np-1}}\,\bigl(\xpt[0]{i} - \xbar_0\bigr)
    \quad\text{where}\quad
    \xbar_0 = \frac{1}{\np}\sum_{i}\xpt[0]{i}.
    \label{eq:coupled_prior}
\end{equation}
Then each pair $i \neq j$ is negatively correlated with $\mathrm{Cov}\bigl(\xap{i}, \xap{j}\bigr) = -\tfrac{1}{\np-1} I_\dd$ while each particle's marginal is preserved.
\end{claim}

\begin{proof}
Since $\mathrm{Cov}(\xpt[0]{i}, \xbar_0) = \mathrm{Cov}(\xbar_0, \xbar_0) = \tfrac{1}{\np} I_\dd$, bilinearity gives $\mathrm{Cov}\bigl(\xpt[0]{i} - \xbar_0,\, \xpt[0]{j} - \xbar_0\bigr) = \bigl(\delta_{ij} - \tfrac{1}{\np}\bigr) I_\dd$, which the factor $\tfrac{\np}{\np-1}$ rescales to $I_\dd$ for $i = j$ and to $-\tfrac{1}{\np-1} I_\dd$ for $i \neq j$.
Each $\xap{i}$ is a linear image of the jointly Gaussian $(\xpt[0]{1},\ldots,\xpt[0]{\np})$ and has mean zero, so its marginal is exactly $\N(0, I_\dd) = p_0$.
\end{proof}

Note that for $\np = 2$ it reduces to $\xap{2} = -\xap{1}$, the classical antithetic-sampling pair.

Because each $\xap{i}$ is exactly $p_0$-distributed, the marginals are exact at $t=0$ and depart from $p_t$ only through the accumulated marginal bias. The per-particle evaluation protocols of Section~\ref{sec:exp_t2i} therefore measure that departure rather than assume it away; Section~\ref{sec:exp_2d} reports its size where the marginal is analytically accessible.

\subsubsection{Complexity}
\label{app:complexity}

Assume a kernel call (forward or backward) takes $\mathcal{O}(1)$, and that a forward and a backward pass through the feature extractor take $F$ and $B$ respectively\footnote{These assumptions were taken with the RBF kernel and DINO in mind.}.
Then \eddy{} (Algorithm~\ref{alg:eddy}) costs $\mathcal{O}\bigl(\np(mF + B) + \np^2\bigr)$ per iteration for the entire batch.

To see this, note that both estimators perturb only $\xp{i}$: the Hutchinson probes $r_1,\ldots,r_m$ of~\eqref{eq:lap_hutch} are drawn independently of $j$, so the $2m+1$ perturbed feature evaluations of particle $i$ are shared across all $\np - 1$ of its neighbours, and the kernel gradients required by~\eqref{eq:hvp_fd} are obtained from two backward passes, giving $\mathcal{O}(mF + B)$ feature-extractor work per particle.
Every remaining quantity in~\eqref{eq:eddy_decomp} is an $\mathcal{O}(1)$ kernel call or the equivalent vector algebra, evaluated once per ordered pair, which contributes $\mathcal{O}(\np^2)$.
Note that the quadratic term involves no network evaluation, so at the particle counts we use the cost is dominated by the $\np(mF + B)$ passes and is effectively linear in $\np$.

\subsubsection{Full Procedure}

\begin{algorithm}[h]
\caption{\eddy{} Sampling}
\label{alg:eddy}
\KwIn{drift $\mu_t$, volatility $\sigma_t$, kernel $k$; steps $T$ with $dt = 1/T$; FD step $\fdeps$, Hutchinson probes $m$; stop ratio $sr$; particle count $\np$; scaling of \eddy{}'s correction term $\wg$}
\KwOut{$\np$ approximate samples from $p_1$}
\BlankLine
Sample $z^1,\ldots,z^\np \overset{\mathrm{iid}}{\sim} \N(\mathbf{0}, I_{\dd})$\;

Set $\xpt[0]{i} \gets 
    \sqrt{\frac{n}{n-1}}
    \left(
        z^i - \frac{1}{\np}\sum_{j=1}^{\np} z^j
    \right)
$ for every $i$\;

\For{$t \gets 0$ \KwTo $1-dt$ \KwBy $dt$}{
    \For{$i = 1$ \KwTo $\np$}{
        \eIf{$t < sr$}{
        Compute $\psii_{\eddy{}} \;=\; \frac{1}{\np-1}\sum_{j\neq i}\bigl(\Aj\,\si \;+\; \Kij\,\vj\bigr)$; use the FD estimators in Eq. ~\ref{eq:hvp_fd} and Eq. ~\ref{eq:lap_hutch}\ if necessary;
        }{
            $\psii_{\eddy{}} \gets \mathbf{0}$\;
        }
        $\xpt[t+dt]{i} \gets$
        SDE step (Eq. ~\ref{eq:forward_sde}) with drift $\muj[i] + \wg \cdot \psii_{\eddy{}}$\;
    }
}
\Return{$\xpt[1]{1},\ldots,\xpt[1]{\np}$}\;
\end{algorithm}

\subsection{\eddyrbf{} Closed-Form Formula}
\label{app:eddy_rbf_cff}
Define $\dij = \xp{i} - \xp{j}$. For the RBF kernel $\kij = \exp\bigl(-\|\dij\|^2/\gamma\bigr)$ with bandwidth $\gamma$, the kernel quantities admit closed forms,
\begin{align}
    \rij \;&=\; -\nabla \kij \;=\; \frac{2}{\gamma}\,\kij\,\dij, \\
    \nabla \rij \;&=\; -\nabla^2 \kij \;=\; \frac{2}{\gamma}\,\kij\!\left(I_{\dd} - \frac{2}{\gamma}\,\dij{\dij}^\top\right), \\
    \divergence \rij \;&=\; -\Delta \kij \;=\; \frac{2}{\gamma}\,\kij\!\left(\dd - \frac{2}{\gamma}\|\dij\|^2\right),
\end{align}
substituted into~\eqref{eq:eddy_decomp}, give the closed-form variant we call \eddyrbf{},
\begin{equation}
    \psii_{\eddy{}} \;=\; \frac{2}{\gamma(\np-1)}\sum_{j\neq i} \kij\bigl(C^\delta_{ij}\,\dij \;+\; C^v_{ij}\,\vj\bigr),
    \label{eq:eddy_rbf}
\end{equation}
with coefficients
\begin{align}
    C^\delta_{ij} \;&=\; \tfrac{2}{\gamma}\langle \dij, \vj\rangle - \langle \vj, \si \rangle, \\
    C^v_{ij} \;&=\; \dd - 1 - \tfrac{2}{\gamma}\|\dij\|^2 + \langle\dij , \si \rangle.
\end{align}

\clearpage
\section{Experiments}

\subsection{Synthetic ODE Results}
\label{app:synthetic}
In the synthetic experiments, we transport a unit Gaussian $\N(0,I_2)$ to an equal mixture of five Gaussians with unit variance where the $\ell^\text{th}$ Gaussian is centered at $5 \cdot \bigl(\sin(2\pi\ell/5),\,\cos(2\pi\ell/5)\bigr)$.
Specifically, we use an optimal transport model:
\begin{equation}
    \mu_\text{ODE}(x_t) = \E_{x_0\sim p_0,\,x_1\sim p_1} [x_1 - x_0 \mid x_t]
\end{equation}
We simulate 5,000 batches of 5 particles, using 10,000 uniform time steps from $t=0$ to $t=1$.
We present goodness-of-fit statistical tests in Table~\ref{tab:gmm}

\begin{figure}[H]
    \centering
    \includegraphics[width=0.5\linewidth]{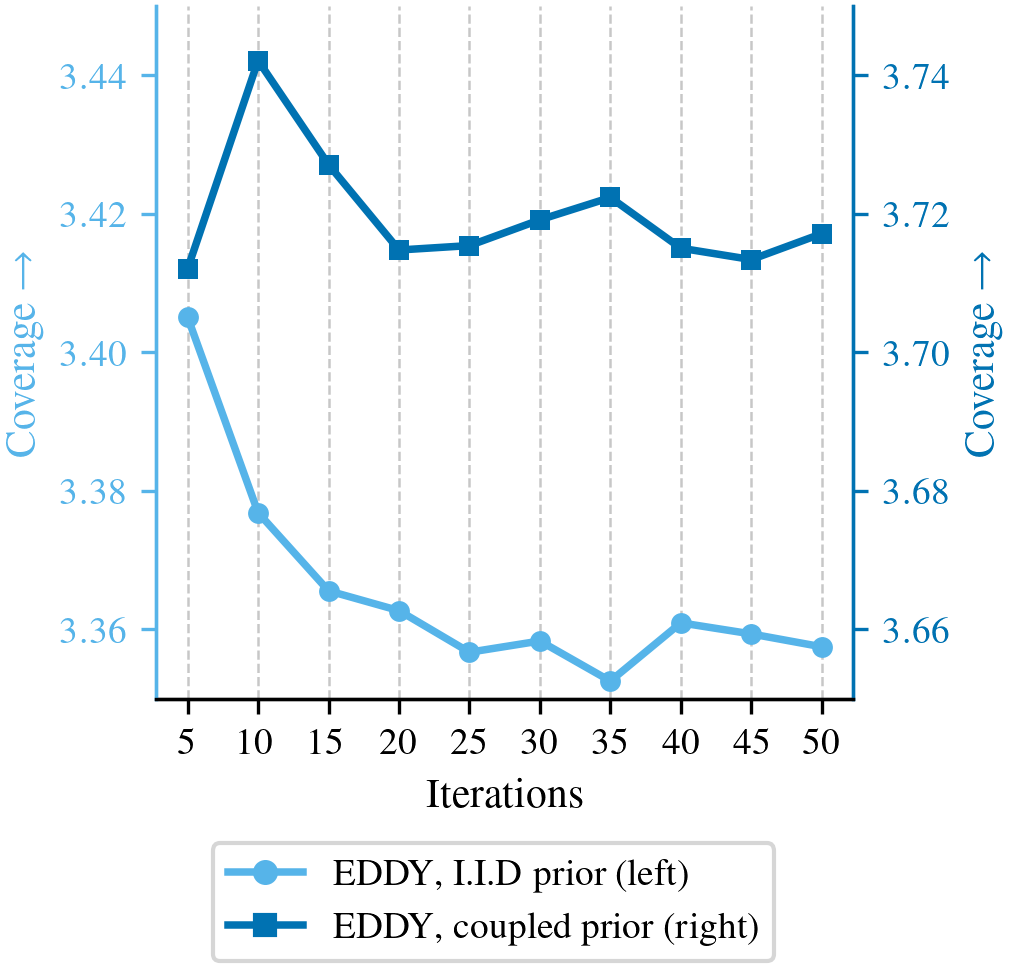}
    \caption{EDDY with i.i.d and coupled prior over increasing number of iterations. It can be seen that EDDY couples i.i.d initialization in discrete time while theory in continuous time prohibit such behavior.}
    \label{fig:eddy_disc}
\end{figure}

\begin{table}[h]
\centering
\small
\caption{
    Goodness-of-fit statistics for the ODE synthetic experiment of Section~\ref{sec:exp_2d} (Fig.~\ref{fig:gmm}). 
    We test whether the radial distance and angular position of \eddy{} and \pg{} samples match \base{} samples using Kolmogorov--Smirnov (KS), Mann--Whitney (MW), and Welch tests at guidance weights $\wg\in\{1.0,2.0,3.0\}$; bold marks $p<0.05$.
}
\label{tab:gmm}
\begin{tabular}{@{} lll rr rr rr @{}}
\toprule
Metric & Method & Test & \multicolumn{2}{c}{$w_g=1.0$} & \multicolumn{2}{c}{$w_g=2.0$} & \multicolumn{2}{c}{$w_g=3.0$} \\
       &        &      & \textit{stat} & \textit{p} & \textit{stat} & \textit{p} & \textit{stat} & \textit{p} \\
\midrule
Distance & \eddy{} & \textsc{KS} & $0.013$ & $0.776$ & $0.013$ & $0.807$ & $0.022$ & $0.194$ \\
 &  & \textsc{MW} & $0.497$ & $0.663$ & $0.499$ & $0.912$ & $0.491$ & $0.126$ \\
 &  & Welch & $-0.337$ & $0.736$ & $-0.031$ & $0.975$ & $-1.346$ & $0.178$ \\
\cmidrule(l){2-9}
 & \pg{} & \textsc{KS} & $0.024$ & $0.123$ & $0.024$ & $0.118$ & $0.019$ & $0.304$ \\
 &  & \textsc{MW} & $0.492$ & $0.143$ & $0.493$ & $0.225$ & $0.493$ & $0.258$ \\
 &  & Welch & $-1.188$ & $0.235$ & $-0.953$ & $0.341$ & $-0.656$ & $0.512$ \\
\midrule
Angle & \eddy{} & \textsc{KS} & $0.012$ & $0.837$ & $0.011$ & $0.942$ & $0.015$ & $0.661$ \\
 &  & \textsc{MW} & $0.502$ & $0.773$ & $0.497$ & $0.648$ & $0.497$ & $0.650$ \\
 &  & Welch & $0.287$ & $0.774$ & $-0.459$ & $0.647$ & $-0.457$ & $0.648$ \\
\cmidrule(l){2-9}
 & \pg{} & \textsc{KS} & $0.021$ & $0.239$ & $0.014$ & $0.678$ & $0.020$ & $0.249$ \\
 &  & \textsc{MW} & $0.506$ & $0.311$ & $0.504$ & $0.509$ & $0.500$ & $0.965$ \\
 &  & Welch & $1.011$ & $0.312$ & $0.660$ & $0.509$ & $0.016$ & $0.987$ \\
\bottomrule
\end{tabular}

\end{table}

\subsection{Text-to-Image Experiment Details and Results}
\label{app:t2i_results}
For the text-to-image experiment, we generate 2048 randomly selected MS-COCO prompts \citep{lin_microsoft_2014} using three backbones.
Specifically, we use FLUX.1-dev and FLUX.2-klein-base-9B for OT-FM, and SDXL for LDM.
All generated images are of resolution 512 $\times$ 512, where in SDXL case we generate 1024 $\times$ 1024 and downsample to match the its training distribution.
Additionally, to implement \eddy{}'s kernel we compose \texttt{facebook/dinov2-with-registers-small} over a fast (tiny) decoder.
We list the per backbone parameters and fast decoders in Table~\ref{tab:base_hparam}.

For each guidance algorithm we perturb a single parameter to control the diversity--quality tradeoff and hold the rest fixed, while applying guidance only during the first $20\%$ of the inference iterations.
We perturb the guidance weight $\wg$ for \eddy{} and \pg{}, and the coefficient for \df{}.
The remaining parameters are common to all diversity levels and backbones: \eddy{} uses $m = 5$ Hutchinson trace samples, $\gamma = 0.05,\epsilon=0.00046$ for FLUX.2, $\gamma = 0.025,\epsilon=0.001$ for FLUX.1 and SDXL; \pg{} has no further parameters; and \df{} uses bandwidth multiplier $h = 0.1$, norm budget $p = 2.0$ and quality clamp $\varepsilon = 0.01$.
We could not find a configuration matching the target diversity tiers for \df{} on SDXL, and therefore omit it from that backbone.
Appendix~\ref{app:cads} reports \cads{} separately, together with its hyperparameters.
We list the perturbed parameter for each algorithm, backbone and diversity tier in Table~\ref{tab:chosen_hparams}.

The evaluation are reported in Table~\ref{tab:pareto_flux}, Table~\ref{tab:pareto_flux2} and Table~\ref{tab:pareto_sdxl}.
FID and CMMD are computed as the mean over $\ell$ of variation $\ell$ of the compared algorithm against variation $(\ell{+}1)\bmod \np$ of \base{}.

\begin{table}[H]
\centering
\caption{Base model generation parameters. The guidance scale refers to CFG in the SDXL case.}
\label{tab:base_hparam}
\begin{tabularx}{\textwidth}{@{} X c c c c @{}}
\toprule
\textbf{Model} & \textbf{Inference steps} & \textbf{Guidance scale} & \textbf{Fast Decoder} \\
\midrule
FLUX.1-dev & 30 & 3.5 & \texttt{madebyollin/taef1} \\
FLUX.2-klein-base-9B & 28 & 3.5 & \texttt{madebyollin/taef2} \\
SDXL & 20 & 7.5 & \texttt{madebyollin/taesdxl} \\
\bottomrule
\end{tabularx}
\end{table}

\begin{table}[H]
\centering
\caption{
    Chosen hyperparameters per backbone, for each method and diversity level (low, medium and high).}
\label{tab:chosen_hparams}
\small
\setlength{\tabcolsep}{4pt}
\begin{tabular}{@{} l ccc @{\hspace{2.3em}} ccc @{\hspace{2.3em}} ccc @{}}
\toprule
& \multicolumn{3}{c}{\eddy{} $(\wg)$} & \multicolumn{3}{c}{\pg{} $(\wg)$} & \multicolumn{3}{c}{\df{} $(\wg)$} \\
\cmidrule(lr){2-4} \cmidrule(lr){5-7} \cmidrule(lr){8-10}
\textbf{Model} & L & M & H & L & M & H & L & M & H \\
\midrule
FLUX.1 & $45$  & $80$  & $115$  & $4000$  & $5000$  & $6000$  & $1500$  & $2000$  & $2500$ \\
FLUX.2 & $50$  & $130$  & $160$  & $40000$ & $50000$ & $60000$ & $10000$ & $12000$ & $14000$ \\
SDXL   & $80$ & $120$ & $160$ & $1200$  & $1600$  & $2000$  & --      & --      & -- \\
\bottomrule
\end{tabular}
\end{table}

\begin{table}[H]
    \centering
    \small
    \caption{
        Pareto frontiers in (a) FLUX.2 9B Base, (b) SDXL and (c) FLUX.1-dev. DINO Sim represent similarity between pairwise samples (lower = more diverse). FID and CMMD represent quality (lower = higher quality). CLIPScore represent prompt alignment (higher = more aligned).
    }
    \begin{subtable}{\linewidth}
        \begin{tabular}{@{} ll ccccc @{}}
\toprule
& \textbf{Method} & \textbf{DINO Sim} $\downarrow$ & \textbf{CMMD} ($\times 10^{-4}$) $\downarrow$ & \textbf{FID} $\downarrow$ & \textbf{CLIP Score} $\uparrow$ & \textbf{Aesthetic} $\uparrow$ \\
\midrule
& \base{} & $0.693{\scriptstyle\,\pm\,0.004}$ & $0.02{\scriptstyle\,\pm\,0.00}$ & $24.562{\scriptstyle\,\pm\,0.039}$ & $27.066{\scriptstyle\,\pm\,0.068}$ & $5.934{\scriptstyle\,\pm\,0.010}$ \\
\midrule
\multirow{3}{*}{Low} & \eddy{} (Ours) & $0.637{\scriptstyle\,\pm\,0.004}$ & $\mathbf{0.04}{\scriptstyle\,\pm\,0.01}$ & $\mathbf{26.127}{\scriptstyle\,\pm\,0.160}$ & $\mathbf{26.840}{\scriptstyle\,\pm\,0.040}$ & $\mathbf{5.870}{\scriptstyle\,\pm\,0.010}$ \\
 & \pg{} & $0.588{\scriptstyle\,\pm\,0.004}$ & $2.56{\scriptstyle\,\pm\,0.03}$ & $43.121{\scriptstyle\,\pm\,0.760}$ & $26.384{\scriptstyle\,\pm\,0.062}$ & $5.381{\scriptstyle\,\pm\,0.010}$ \\
 & \df{} & $0.599{\scriptstyle\,\pm\,0.004}$ & $3.06{\scriptstyle\,\pm\,0.02}$ & $46.524{\scriptstyle\,\pm\,0.391}$ & $26.415{\scriptstyle\,\pm\,0.063}$ & $5.425{\scriptstyle\,\pm\,0.010}$ \\
\midrule
\multirow{3}{*}{Medium} & \eddy{} (Ours) & $0.535{\scriptstyle\,\pm\,0.005}$ & $\mathbf{0.45}{\scriptstyle\,\pm\,0.01}$ & $\mathbf{33.370}{\scriptstyle\,\pm\,0.165}$ & $25.941{\scriptstyle\,\pm\,0.048}$ & $\mathbf{5.683}{\scriptstyle\,\pm\,0.011}$ \\
 & \pg{} & $0.532{\scriptstyle\,\pm\,0.004}$ & $4.77{\scriptstyle\,\pm\,0.08}$ & $62.468{\scriptstyle\,\pm\,0.686}$ & $25.830{\scriptstyle\,\pm\,0.060}$ & $5.066{\scriptstyle\,\pm\,0.011}$ \\
 & \df{} & $0.559{\scriptstyle\,\pm\,0.004}$ & $5.76{\scriptstyle\,\pm\,0.05}$ & $70.456{\scriptstyle\,\pm\,0.492}$ & $\mathbf{25.960}{\scriptstyle\,\pm\,0.061}$ & $5.123{\scriptstyle\,\pm\,0.010}$ \\
\midrule
\multirow{3}{*}{High} & \eddy{} (Ours) & $0.495{\scriptstyle\,\pm\,0.005}$ & $\mathbf{0.73}{\scriptstyle\,\pm\,0.03}$ & $\mathbf{37.221}{\scriptstyle\,\pm\,0.703}$ & $\mathbf{25.489}{\scriptstyle\,\pm\,0.052}$ & $\mathbf{5.622}{\scriptstyle\,\pm\,0.012}$ \\
 & \pg{} & $0.462{\scriptstyle\,\pm\,0.004}$ & $7.92{\scriptstyle\,\pm\,0.08}$ & $92.031{\scriptstyle\,\pm\,1.787}$ & $24.821{\scriptstyle\,\pm\,0.060}$ & $4.646{\scriptstyle\,\pm\,0.012}$ \\
 & \df{} & $0.512{\scriptstyle\,\pm\,0.004}$ & $9.54{\scriptstyle\,\pm\,0.05}$ & $106.481{\scriptstyle\,\pm\,1.007}$ & $25.169{\scriptstyle\,\pm\,0.061}$ & $4.773{\scriptstyle\,\pm\,0.010}$ \\
\bottomrule
\end{tabular}

        \subcaption{Pareto frontiers on FLUX.2 9B Base.}
        \label{tab:pareto_flux2}
    \end{subtable}
    \begin{subtable}{\linewidth}
        \begin{tabular}{@{} ll ccccc @{}}
\toprule
& \textbf{Method} & \textbf{DINO Sim} $\downarrow$ & \textbf{CMMD} ($\times 10^{-4}$) $\downarrow$ & \textbf{FID} $\downarrow$ & \textbf{CLIP Score} $\uparrow$ & \textbf{Aesthetic} $\uparrow$ \\
\midrule
& \base{} & $0.666{\scriptstyle\,\pm\,0.004}$ & $0.02{\scriptstyle\,\pm\,0.00}$ & $25.595{\scriptstyle\,\pm\,0.073}$ & $26.886{\scriptstyle\,\pm\,0.063}$ & $6.227{\scriptstyle\,\pm\,0.009}$ \\
\midrule
\multirow{2}{*}{Low} & \eddy{} (Ours) & $0.619{\scriptstyle\,\pm\,0.004}$ & $\mathbf{0.19}{\scriptstyle\,\pm\,0.01}$ & $\mathbf{28.787}{\scriptstyle\,\pm\,0.262}$ & $\mathbf{26.467}{\scriptstyle\,\pm\,0.036}$ & $\mathbf{6.075}{\scriptstyle\,\pm\,0.010}$ \\
 & \pg{} & $0.611{\scriptstyle\,\pm\,0.004}$ & $1.58{\scriptstyle\,\pm\,0.04}$ & $30.131{\scriptstyle\,\pm\,0.216}$ & $26.271{\scriptstyle\,\pm\,0.061}$ & $5.898{\scriptstyle\,\pm\,0.011}$ \\
\midrule
\multirow{2}{*}{Medium} & \eddy{} (Ours) & $0.569{\scriptstyle\,\pm\,0.004}$ & $\mathbf{0.53}{\scriptstyle\,\pm\,0.01}$ & $\mathbf{35.432}{\scriptstyle\,\pm\,0.302}$ & $\mathbf{26.091}{\scriptstyle\,\pm\,0.041}$ & $\mathbf{6.003}{\scriptstyle\,\pm\,0.010}$ \\
 & \pg{} & $0.551{\scriptstyle\,\pm\,0.005}$ & $3.56{\scriptstyle\,\pm\,0.10}$ & $39.253{\scriptstyle\,\pm\,0.563}$ & $25.587{\scriptstyle\,\pm\,0.063}$ & $5.658{\scriptstyle\,\pm\,0.012}$ \\
\midrule
\multirow{2}{*}{High} & \eddy{} (Ours) & $0.519{\scriptstyle\,\pm\,0.004}$ & $\mathbf{1.17}{\scriptstyle\,\pm\,0.06}$ & $\mathbf{46.066}{\scriptstyle\,\pm\,1.373}$ & $\mathbf{25.603}{\scriptstyle\,\pm\,0.047}$ & $\mathbf{5.853}{\scriptstyle\,\pm\,0.011}$ \\
 & \pg{} & $0.486{\scriptstyle\,\pm\,0.004}$ & $5.92{\scriptstyle\,\pm\,0.06}$ & $53.471{\scriptstyle\,\pm\,0.351}$ & $24.762{\scriptstyle\,\pm\,0.066}$ & $5.413{\scriptstyle\,\pm\,0.014}$ \\
\bottomrule
\end{tabular}

        \subcaption{Pareto frontiers on SDXL.}
        \label{tab:pareto_sdxl}
    \end{subtable}
    \begin{subtable}{\linewidth}
        \begin{tabular}{@{} ll ccccc @{}}
\toprule
& \textbf{Method} & \textbf{DINO Sim} $\downarrow$ & \textbf{CMMD} ($\times 10^{-4}$) $\downarrow$ & \textbf{FID} $\downarrow$ & \textbf{CLIP Score} $\uparrow$ & \textbf{Aesthetic} $\uparrow$ \\
\midrule
& \base{} & $0.707{\scriptstyle\,\pm\,0.004}$ & $0.01{\scriptstyle\,\pm\,0.00}$ & $20.530{\scriptstyle\,\pm\,0.105}$ & $25.610{\scriptstyle\,\pm\,0.071}$ & $6.430{\scriptstyle\,\pm\,0.009}$ \\
\midrule
\multirow{3}{*}{Low} & \eddy{} (Ours) & $0.645{\scriptstyle\,\pm\,0.004}$ & $\mathbf{0.16}{\scriptstyle\,\pm\,0.01}$ & $\mathbf{23.049}{\scriptstyle\,\pm\,0.300}$ & $\mathbf{25.549}{\scriptstyle\,\pm\,0.038}$ & $\mathbf{6.278}{\scriptstyle\,\pm\,0.010}$ \\
 & \pg{} & $0.627{\scriptstyle\,\pm\,0.004}$ & $2.11{\scriptstyle\,\pm\,0.04}$ & $30.822{\scriptstyle\,\pm\,0.425}$ & $24.927{\scriptstyle\,\pm\,0.066}$ & $6.013{\scriptstyle\,\pm\,0.010}$ \\
 & \df{} & $0.646{\scriptstyle\,\pm\,0.004}$ & $10.00{\scriptstyle\,\pm\,0.14}$ & $51.423{\scriptstyle\,\pm\,0.624}$ & $24.182{\scriptstyle\,\pm\,0.065}$ & $5.467{\scriptstyle\,\pm\,0.010}$ \\
\midrule
\multirow{3}{*}{Medium} & \eddy{} (Ours) & $0.564{\scriptstyle\,\pm\,0.005}$ & $\mathbf{0.92}{\scriptstyle\,\pm\,0.03}$ & $\mathbf{30.809}{\scriptstyle\,\pm\,0.748}$ & $\mathbf{25.240}{\scriptstyle\,\pm\,0.042}$ & $\mathbf{5.968}{\scriptstyle\,\pm\,0.013}$ \\
 & \pg{} & $0.568{\scriptstyle\,\pm\,0.004}$ & $4.86{\scriptstyle\,\pm\,0.08}$ & $46.133{\scriptstyle\,\pm\,0.659}$ & $24.418{\scriptstyle\,\pm\,0.064}$ & $5.719{\scriptstyle\,\pm\,0.011}$ \\
 & \df{} & $0.588{\scriptstyle\,\pm\,0.004}$ & $15.94{\scriptstyle\,\pm\,0.10}$ & $88.871{\scriptstyle\,\pm\,0.747}$ & $23.683{\scriptstyle\,\pm\,0.061}$ & $4.777{\scriptstyle\,\pm\,0.010}$ \\
\midrule
\multirow{3}{*}{High} & \eddy{} (Ours) & $0.493{\scriptstyle\,\pm\,0.004}$ & $\mathbf{2.19}{\scriptstyle\,\pm\,0.06}$ & $\mathbf{41.036}{\scriptstyle\,\pm\,0.270}$ & $\mathbf{24.779}{\scriptstyle\,\pm\,0.044}$ & $\mathbf{5.683}{\scriptstyle\,\pm\,0.014}$ \\
 & \pg{} & $0.495{\scriptstyle\,\pm\,0.004}$ & $8.54{\scriptstyle\,\pm\,0.23}$ & $69.051{\scriptstyle\,\pm\,1.328}$ & $23.792{\scriptstyle\,\pm\,0.065}$ & $5.362{\scriptstyle\,\pm\,0.012}$ \\
 & \df{} & $0.534{\scriptstyle\,\pm\,0.003}$ & $19.62{\scriptstyle\,\pm\,0.08}$ & $128.945{\scriptstyle\,\pm\,0.360}$ & $23.327{\scriptstyle\,\pm\,0.062}$ & $4.287{\scriptstyle\,\pm\,0.009}$ \\
\bottomrule
\end{tabular}

        \subcaption{Pareto frontiers on FLUX.1-dev.}
        \label{tab:pareto_flux}
    \end{subtable}
\end{table}

\subsection{Comparison to \cads{}}
\label{app:cads}
\cads{} \citep{sadat_cads_2024} induces diversity by a different mechanism than the methods compared in Section~\ref{sec:exp_t2i}.
Rather than coupling the particles of a batch, it anneals noise into the conditioning signal itself, so the samples remain mutually independent and diversity arises from each one being generated under a slightly different condition.
This makes the diversity--quality trade-off it traces qualitatively different: the quantity it spends is not image fidelity but prompt adherence.
We therefore report it here rather than alongside \pg{} and \df{}.

We tune the noise scale to the same three diversity tiers used in the main comparison, holding $\tau_1 = 0.7$, $\tau_2 = 0.8$, rescale $= \mathrm{False}$ and guidance scale $g = 5.0$ fixed; the chosen noise scales on FLUX.1-dev are $0.05$, $0.07$ and $0.08$.
Table~\ref{tab:cads_flux} compares \cads{} against \eddy{} at matched diversity on FLUX.1-dev.

\cads{} attains strong image-quality scores --- at medium and high diversity it beats \eddy{} on CMMD, FID and Aesthetic, while its CLIPScore fall significantly from \eddy{} and \base{}.
Figure~\ref{fig:cheff_grid} shows this directly: for the prompt \emph{``a chef’s display of ingredients and pastry creations''}, the \cads{} samples are photorealistic but showing no chef, and in some photos no pastries.

On SDXL and FLUX.2-klein-base-9B we could not reproduce even this trade-off.
Despite extensive hyperparameter tuning, every configuration we found either yielded high diversity at very low prompt adherence or left diversity essentially unchanged from the baseline, with no usable regime in between.
We therefore evaluate \cads{} on FLUX.1-dev only.

\begin{table}[H]
    \centering
    \small
    \caption{
        \cads{} versus \eddy{} at matched diversity tiers on FLUX.1-dev.
        Bold marks the better of the two per metric.
        DINO Sim represents similarity between pairwise samples (lower = more diverse); FID and CMMD represent quality (lower = higher quality); CLIPScore represents prompt alignment (higher = more aligned).
    }
    \label{tab:cads_flux}
    \begin{tabular}{@{} ll ccccc @{}}
\toprule
& \textbf{Method} & \textbf{DINO Sim} $\downarrow$ & \textbf{CMMD} ($\times 10^{-4}$) $\downarrow$ & \textbf{FID} $\downarrow$ & \textbf{CLIP Score} $\uparrow$ & \textbf{Aesthetic} $\uparrow$ \\
\midrule
& \base{} & $0.707{\scriptstyle\,\pm\,0.004}$ & $0.01{\scriptstyle\,\pm\,0.00}$ & $20.530{\scriptstyle\,\pm\,0.105}$ & $25.610{\scriptstyle\,\pm\,0.071}$ & $6.430{\scriptstyle\,\pm\,0.009}$ \\
\midrule
\multirow{2}{*}{Low} & \eddy{} & $0.645{\scriptstyle\,\pm\,0.004}$ & $\mathbf{0.16}{\scriptstyle\,\pm\,0.01}$ & $\mathbf{23.049}{\scriptstyle\,\pm\,0.300}$ & $\mathbf{25.549}{\scriptstyle\,\pm\,0.038}$ & $6.278{\scriptstyle\,\pm\,0.010}$ \\
 & \cads{} & $0.618{\scriptstyle\,\pm\,0.005}$ & $0.31{\scriptstyle\,\pm\,0.01}$ & $25.032{\scriptstyle\,\pm\,0.107}$ & $23.164{\scriptstyle\,\pm\,0.068}$ & $\mathbf{6.563}{\scriptstyle\,\pm\,0.010}$ \\
\midrule
\multirow{2}{*}{Medium} & \eddy{} & $0.564{\scriptstyle\,\pm\,0.005}$ & $0.92{\scriptstyle\,\pm\,0.03}$ & $30.809{\scriptstyle\,\pm\,0.748}$ & $\mathbf{25.240}{\scriptstyle\,\pm\,0.042}$ & $5.968{\scriptstyle\,\pm\,0.013}$ \\
 & \cads{} & $0.575{\scriptstyle\,\pm\,0.005}$ & $\mathbf{0.57}{\scriptstyle\,\pm\,0.01}$ & $\mathbf{27.876}{\scriptstyle\,\pm\,0.271}$ & $22.530{\scriptstyle\,\pm\,0.074}$ & $\mathbf{6.609}{\scriptstyle\,\pm\,0.010}$ \\
\midrule
\multirow{2}{*}{High} & \eddy{} & $0.493{\scriptstyle\,\pm\,0.004}$ & $2.19{\scriptstyle\,\pm\,0.06}$ & $41.036{\scriptstyle\,\pm\,0.270}$ & $\mathbf{24.779}{\scriptstyle\,\pm\,0.044}$ & $5.683{\scriptstyle\,\pm\,0.014}$ \\
 & \cads{} & $0.522{\scriptstyle\,\pm\,0.005}$ & $\mathbf{0.80}{\scriptstyle\,\pm\,0.02}$ & $\mathbf{31.420}{\scriptstyle\,\pm\,0.108}$ & $21.584{\scriptstyle\,\pm\,0.080}$ & $\mathbf{6.615}{\scriptstyle\,\pm\,0.010}$ \\
\bottomrule
\end{tabular}

\end{table}

\subsection{Text-to-Image Ablation Study}
\label{app:ablation}
\begin{table}[H]
\centering
\caption{Wall-clock overhead (seconds per generation) for each guidance configuration on a NVIDIA A100 40GB PCIe GPU.}
\label{tab:elapsed_time}
\begin{tabular}{@{} l r r r @{}}
\toprule
Method & \textbf{SDXL} & \textbf{Flux.1} & \textbf{Flux.2} \\
\midrule
\base{} & \texttt{8.0 $\pm$ 0.0}\,s & \texttt{25.0 $\pm$ 0.5}\,s & \texttt{40.1 $\pm$ 0.4}\,s \\
\pg{} & \texttt{9.0 $\pm$ 0.0}\,s & \texttt{25.9 $\pm$ 0.6}\,s & \texttt{41.1 $\pm$ 0.3}\,s \\
\df{} & -- & \texttt{25.1 $\pm$ 0.6}\,s & \texttt{40.9 $\pm$ 0.4}\,s \\
\eddy{} & \texttt{14.0 $\pm$ 0.0}\,s & \texttt{28.7 $\pm$ 0.7}\,s & \texttt{44.0 $\pm$ 0.0}\,s \\
\bottomrule
\end{tabular}

\end{table}

\begin{table}[H]
\centering
\caption{Wall-clock overhead (seconds per generation) for EDDY for varied number of particles on a NVIDIA A100 40GB PCIe GPU.}
\label{tab:elapsed_time_vs_n}
\begin{tabular}{@{} l l r r r @{}}
\toprule
$\np$ & \textbf{SDXL} & \textbf{Flux.1} & \textbf{Flux.2} \\
\midrule
3 & \texttt{9.9 $\pm$ 0.4}\,s & \texttt{23.2 
$\pm$ 0.4}\,s & \texttt{34.1 $\pm$ 0.3}\,s \\
4 & \texttt{14.0 $\pm$ 0.0}\,s & \texttt{28.7 $\pm$ 0.7}\,s & \texttt{44.0 $\pm$ 0.0}\,s \\
5 & \texttt{19.9 $\pm$ 0.3}\,s & \texttt{35.3 $\pm$ 0.7}\,s & \texttt{53.0 $\pm$ 0.0}\,s \\
6 & \texttt{25.0 $\pm$ 0.0}\,s & \texttt{41.1 $\pm$ 0.5}\,s & \texttt{63.0 $\pm$ 0.0}\,s \\
\bottomrule
\end{tabular}

\end{table}

\begin{table}[H]
\centering
\caption{
Change in each metric relative to \base{} when prior is coupled but no guidance applied, with 512 MS-COCO prompts tested. Bold: $p<0.05$,
Welch $z$-test.}
\label{tab:coupling_only}
\begin{tabular}{@{} l rrr @{}}
\toprule
Backbone & DINO pairwise sim & Aesthetic score & CLIPScore \\
\midrule
SDXL & $-0.0041$ & $\mathbf{-0.1338}$ & $+0.0126$ \\
FLUX.1 & $-0.0001$ & $-0.0140$ & $+0.1684$ \\
FLUX.2 & $-0.0005$ & $-0.0101$ & $-0.0046$ \\
\bottomrule
\end{tabular}

\end{table}

\clearpage
\section{Score Derivation via Tweedie's formula in Flow Matching}
\label{app:tweedie}
First, let us state Tweedie's formula \citep{efron_tweedies_2011}: given clean $x$ and estimation $\hat x = x + \epsilon$ with independent noise $\epsilon \sim \N(0, \sigma^2 I_\dd)$ it follows that $\E[x \mid \hat x] = \hat x + \sigma^2\, \score(\hat x)$.

In training, a flow matching model is trained to predict $\xt[1] - \xt[0]$ given $\xt = t\xt[1] + (1-t)\xt[0]$ for noise $\xt[0] \sim p_0 = \N(0,I_\dd)$ and data point $\xt[1] \sim p_1$. Thus, given a state $\xt$ the model will estimate $\E [\xt[1] - \xt[0] \mid \xt]$ for some unknown $\xt[1]$.

From one hand, using Tweedie's formula over "clean" $t\xt[1]$ and estimation $\xt$ with independent noise $(1-t)\xt[0] \sim \N\bigl(0,(1-t)^2I_\dd\bigr)$, we get
\begin{equation}
    \E[t\xt[1] \mid \xt] = \xt + (1-t)^2\, \score_t(\xt)
\end{equation}
From the other hand, it easily follows from the $\xt$'s definition that $t\xt[1] = t(1-t)(\xt[1] - \xt[0]) + t\xt$, so by taking expectation conditioned on $\xt$ we get
\begin{equation}
    \E[t\xt[1] \mid \xt] = t(1-t)\,\E[\xt[1] - \xt[0]\mid \xt] + t\xt
\end{equation}
Putting it all together
\begin{equation}
    \score_t(\xt) = \frac{t\E[\xt[1]-\xt[0]\mid\xt] - \xt}{1-t}
\end{equation}

\end{document}